\documentclass[twoside,twocolumn,letterpaper]{article}
\usepackage[accepted]{aistats2020}
\usepackage{titlesec}
\titleformat{\section}
  {\large\bfseries}
  {\thesection.\enskip}
  {0pt}
  {\MakeUppercase}
\usepackage[letterpaper,twoside,twocolumn,total={6.75in,9.25in},left=1in,columnsep=0.25in, top=1in]{geometry}
\usepackage[accepted]{aistats2020}
\usepackage{amsthm}
\usepackage{enumitem}
\setitemize{itemsep=0pt,topsep=0pt,leftmargin=*}
\usepackage{multirow}
\usepackage{afterpage}

\usepackage{url}            
\usepackage{booktabs}       
\usepackage{amsfonts}       
\usepackage{amsmath}
\usepackage{nicefrac}       
\usepackage{microtype}      
\usepackage{graphicx}
\usepackage{caption}
\usepackage{subcaption}
\usepackage{xcolor}
\usepackage{algorithm}
\usepackage{graphicx, caption, subcaption}
\usepackage[noend]{algpseudocode}
 \usepackage{url}

\usepackage{cvpr-abbriv}
\usepackage{cvpr-abbriv}
\usepackage[titletoc,title]{appendix}
\usepackage{chngcntr}

\usepackage[draft=false, pagebackref=false,breaklinks=true,colorlinks,bookmarks=false,pdftex]{hyperref}
\hypersetup{hidelinks,colorlinks=false,linkbordercolor={1 1 1}}
\usepackage{thmtools, thm-restate}
\usepackage[capitalise,nameinlink]{cleveref}
\crefformat{equation}{(#2#1#3)}

\crefname{section}{Sec.}{Sec.}
\Crefname{section}{Sec.}{Sec.}

\newcommand{\gray}{\color[rgb]{0.5,0.5,0.5}}
\newcommand{\red}{\color[rgb]{1,0,0}}

\newcommand{\Y}{\mathcal{Y}}
\newcommand{\V}{\mathcal{V}}
\newcommand{\E}{\mathcal{E}}

\newcommand{\G}{\mathcal{G}}

\renewcommand*{\paragraph}[1]{\par\noindent{\normalsize\bf #1}\,\xspace}

\newcommand{\revisit}[1][]{%
\ifthenelse{\equal{#1}{}}{
\ensuremath{\red \triangle}\xspace}{%
{\ensuremath{\red \rhd}\xspace}%
{\gray #1}%
{\ensuremath{\red \lhd}\xspace}%
}%
}

\def\mathrlap{\mathpalette\mathrlapinternal} 
\def\mathllap{\mathpalette\mathllapinternal}
\def\mathllapinternal#1#2{\llap{$\mathsurround=0pt#1{#2}$}}
\def\mathrlapinternal#1#2{\rlap{$\mathsurround=0pt#1{#2}$}}

\def\leftbb{\mathrlap{[}\hskip1.3pt[}
\def\rightbb{]\hskip1.36pt\mathllap{]}}

\theoremstyle{definition}
\newtheorem{definition}{Definition}[]

\newcommand{\floor}[1]{\lfloor #1 \rfloor}

\DeclareMathOperator*\argmin{arg\,min}

\DeclareMathOperator*\nb{Nb}

\DeclareMathOperator*\MP{\texttt{MPLP}}
\DeclareMathOperator*\HS{\texttt{HS}}
\DeclareMathOperator*\DP{\texttt{DP}}
\DeclareMathOperator*\rDP{\texttt{rDP}}

\newcommand\BR{\mathbb R}
\newcommand\SG{\mathcal G}

\newcommand\SV{\mathcal V}
\newcommand\SE{\mathcal E}
\newcommand\SC{\mathcal C}

\newcommand\SY{\mathcal Y}

\newcommand\SQ{\mathcal Q}

\newcommand\MSD{\texttt{MSD}\xspace}
\newcommand\TRWS{\texttt{TRW-S}\xspace}
\newcommand\SRMP{\texttt{SRMP}\xspace}
\newcommand\CMP{\texttt{CMP}\xspace}
\newcommand\MPLP{\texttt{MPLP}\xspace}
\newcommand\MPLPPP{\texttt{MPLP++}\xspace}
\newcommand\DMM{\texttt{DMM}\xspace}

\newcommand\SPAM{\texttt{SPAM}\xspace}
\newcommand\TBCA{\texttt{TBCA}\xspace}
\newcommand\TBCAPP{\texttt{TBCA++}\xspace}
\newcommand\tDP{\texttt{DP}\xspace}
\newcommand\trDP{\texttt{rDP}\xspace}
\newcommand\HM{\texttt{HM}\xspace}

\DeclareMathAlphabet\mathbfcal{OMS}{cmsy}{b}{n}

\usepackage[square,sort,comma,numbers]{natbib}

\begin{document}

%

%
\runningauthor{Tourani, Shekhovtsov, Rother, Savchynskyy}

\twocolumn[

\aistatstitle{Taxonomy of Dual Block-Coordinate Ascent Methods for Discrete Energy Minimization}

\aistatsauthor{ Siddharth Tourani$^1$ \And Alexander Shekhovtsov$^2$ \And  Carsten Rother$^1$ \And Bogdan Savchynskyy$^1$ }

\aistatsaddress{$^1$University of Heidelberg, Germany \And $^2$Czech Technical University in Prague} ]

\begin{abstract}
We consider the maximum-a-posteriori inference problem in discrete graphical models and study solvers based on the dual block-coordinate ascent rule. We map all existing solvers in a single framework, allowing for a better understanding of their design principles. We theoretically show that some block-optimizing updates are sub-optimal and how to strictly improve them. On a wide range of problem instances of varying graph connectivity, we study the performance of existing solvers as well as new variants that can be obtained within the framework. As a result of this exploration we build a new state-of-the art solver, performing uniformly better on the whole range of test instances.

\end{abstract}

\section{Introduction}

Discrete graphical models, one of the most sound and powerful frameworks in computer vision and machine learning, is still used in many applications in the era of CNNs. Graphical models effectively encode domain specific prior information in the form of a structured cost function, which is often hard to learn from data directly. With an increase in parallelization, fast dual block-coordinate ascent algorithms (BCA) have been developed that allow their application \eg in stereo~\cite{Discrete-Continuous-16}, optical flow~\cite{munda-17-flow}, 6D pose estimation~\cite{Tourani_2018_ECCV}. Combined and jointly trained with CNNs they can create more powerful models~\cite{chen2015learning,knobelreiter2017end}. They can also provide efficient regularization for training of CNN models~\cite{KolesnikovL16,Marin_2019_CVPR,Tosi_2019_CVPR}. Applications where structural constraints must be fulfilled (\eg\ \cite{payer2016automated}) or the optimality is required also significantly benefit from fast computation of good lower bounds by such methods~\cite{savchynskyy2013global, Haller18}.

In this work we systematically review the existing BCA methods. Despite being developed for different dual decompositions, they can be equivalently formulated as BCA methods on the same dual problem. We contribute a theoretical analysis showing which block updates are sub-optimal and can be improved. We perform an experimental study on a corpus of very diverse problem instances to discover important properties relevant to algorithm design. Such as, which types of variable updates are more efficient, or whether a dynamic or static strategy  in sub-problem selection is better, \etc. One observation that we made is that there is currently no single algorithm that would work well for both, sparse and dense problems. With this new comparison and theoretical insights, we synthesize a novel BCA method, that selects subproblems automatically adapting to the given graph structure. 
It applies the type of updates that are more expensive but which turn out to be more efficient and performs universally better across the whole range of the problems in the datasets we used.
\begin{table*}[!t]
\resizebox{\linewidth}{!}{
\small
\setlength{\tabcolsep}{2pt}
\renewcommand*{\arraystretch}{1.1}
\begin{tabular}{|l|l|l|l|}
\multicolumn{1}{c}{\bf \scriptsize Abbreviation} &
\multicolumn{1}{c}{\bf \scriptsize Authors} &
\multicolumn{1}{c}{\bf \scriptsize Method Name} &
\multicolumn{1}{c}{\bf \scriptsize Type of blocks / updates} \\
 \hline
\MSD & \citet{schlesingera2011diffusion} & Min-Sum Diffusion & \multirow{2}{*}{Node-adjacent, isotropic}\\
\CMP &\citet[Alg.5]{hazan2010norm} & Convex Max Product & \\
\hline
\TRWS & \citet{kolmogorov2006convergent} & Tree-Reweighted Message Passing & \multirow{2}{*}{Node-adjacent, anisotropic}\\
\SRMP &\citet{kolmogorov2015new} & Sequential Reweighted Message Passing & \\
\hline
\MPLP &\citet{NIPS2007_3200} & \multirow{2}{*}{Max-Product Linear Programming} & \multirow{2}{*}{Edges} \\
\MPLPPP &\citet{Tourani_2018_ECCV} & & \\
\hline
\DMM &\citet{Discrete-Continuous-16} & Dual Minorize-Maximize & Chains, hierarchical\\
\hline
\multirow{2}{*}{\TBCA} & \citet{sontag2009tree} & Tree Block Coordinate Ascent & Trees, sequential\\
  & \citet{tarlow2011dynamic} & Dynamic Tree Block Coordinate Ascent & Dynamic trees, sequential \\
\hline
\end{tabular}
}%
\caption{Surveyed block-coordinate ascent algorithms.\label{table:algs}}
\end{table*}
\subsection{Related Work}
Inference in graphical models is a well-known NP-hard problem. A number of solvers with different time complexities and guarantees, utilized in different applications, is surveyed in~\cite{kappes-2015-ijcv,Li2016a}. The linear-programming approach and the large family of associated methods is well covered in~\cite{werner2007linear,bogdan2019discretebook}. 
In this work we focus on BCA methods, which appear to offer the best lower bounds with a limited time budget for pairwise models with general pairwise interactions. These methods can be used to obtain fast approximate solutions directly, or to efficiently reduce the full combinatorial search~\cite{savchynskyy2013global, Haller18}.
Many BCA methods have been proposed to date and we selected in~\cref{table:algs} a mostly complete and representative list of the state-of-the-art BCA algorithms. Some of these methods were originally obtained for different dual formulations, based on the decompositions into larger subproblems (\TRWS, \TBCA, \DMM). Although, it is known that these duals are equivalent in the optimum~\cite{wainwright2005map,bogdan2019discretebook}, it has been believed that optimizing a stronger dual can be more efficient. 
Works~\cite{meltzer2005globally,ruozzi2013message} proposed a unified view of several MAP and sum-product algorithms as BCA methods. However, the dual objectives were different per method (derived from different region graphs~\cite{meltzer2005globally}, resp. splittings~\cite{ruozzi2013message}) and the algorithms operate with messages and beliefs. We consider a single dual for all methods, following the more recent understanding of \TRWS~\cite{kolmogorov2015new}, and all algorithms are explicitly updating the same dual variables.	

We study the issue of non-uniqueness of the block maximizers in BCA methods and their influence on the overall algorithmic efficiency.  \citet{Tourani_2018_ECCV} shows that \MPLP method can be significantly improved by a small modification in the choice of block maximizers. We generalize these results to chain and tree subproblems. \citet{Prusa-Werner-arXiv-relint} study the effect on fixed points.

Our code is available at \url{https://gitlab.com/tourani.siddharth/spam-code}. Proofs of all mathematical statements can be found in the appendix.

\section{MAP Inference with BCA}\label{sec:prelims}

\paragraph{MAP-Inference Problem}
Let $\SG=(\SV,\SE)$ be an undirected graph with the {\em node} set $\SV$ and {\em edge} set $\SE$. 
A {\em labeling} $y \colon \SV \to \SY$ assigns to each node $u\in \SV$ a discrete {\em label} $y_u \in \Y$, where $\SY$ is some finite set of labels, \Wlog assumed the same for all nodes. For brevity we will denote edges $\{u,v\}\in\SE$ as just $uv$.

For each node $u\in\SV$ and edge $uv\in\SE$ there are associated the following local cost functions: $\theta_u(s) \geq 0$ is the cost of a label $s \in \Y$ and $\theta_{uv} (s,t) \geq 0$ is the cost of a label pair $(s,t) \in \Y^2$, where the non-negativity is assumed w.l.o.g. Let also $\nb(u)$ denote the set of neighbors of node $u$ in $\SG$. 

In the well-known paradigm of MRF / CRF models, the posterior probability distribution is defined via the {\em energy} $E(y)$ as $p(y) \propto \exp(-E(y))$ and the {\em maximum a posteriori (MAP) inference} problem becomes equivalent to finding a labeling which minimizes the energy (total labeling cost): $y^*=$
\begin{equation}\label{equ:energy-min}
\textstyle
 \argmin\limits_{y\in\Y^\V} \Big[ E(y \mid \theta):=\sum\limits_{v\in\SV}\theta_v(y_v)+\sum\limits_{uv\in\SE}\theta_{uv}(y_{uv})\Big].
\end{equation}
\vskip-0.5\baselineskip
\paragraph{Reparametrizations}
The representation of the energy function $E(y \mid \theta)$ as the sum of unary and pairwise costs is not unique: there exist many cost vectors $\theta'$ such that $E(y \mid \theta)=E(y \mid \theta')$ for all labelings $y\in\SY^\SV$. Such cost vectors are called {\em equivalent}. 
All cost vectors equivalent to $\theta$ can be obtained as (\eg,~\cite{werner2007linear}):
\begin{align}\label{equ:reparametrization}
\textstyle
 \textstyle \theta^\phi_u(s) & \textstyle = \theta_u(s)-\sum_{v\in\nb(u)}\phi_{u,v}(s),\\
 \textstyle \theta^\phi_{uv}(s,t) & \textstyle = \theta_{uv}(s,t) + \phi_{u,v}(s)+\phi_{v,u}(t)\, \nonumber
\end{align}
with some \emph{reparametrization} vector $\phi = (\phi_{u,v}(s)\in\BR \mid u\in\SV,\ v\in\nb(u),\ s\in\SY)$. This reparametrization is illustrated in~\cref{fig:1}(a). It is straightforward to see that when substituting~\eqref{equ:reparametrization} into~\eqref{equ:energy-min} all contributions from $\phi$ cancel out and thus any reparametrized $\theta^\phi$ is equivalent to $\theta$ (for the converse, that all equivalent costs do have such a representation see~\cite{werner2007linear}).

\paragraph{Dual Problem}
The basic idea, pioneered in pattern recognition by~\cite{schlesinger1976syntactic}, is the following. In practice there exist oftentimes a reparametrization with the property that by selecting the label in each node independently as $y_u \in \argmin \theta^\phi_u$ a good, or even optimal, solution is recovered.

From an optimization perspective, this is captured by the lower bound: $D(\phi) :=$
\begin{align}\label{equ:Lagrange-dual}
\sum_{u\in\SV}\min_{s\in\SY}\theta^\phi_u(s) +\hskip-2pt\sum_{uv\in\SE}\min_{(s,t)\in\SY_{uv} \hskip-8pt}\theta^\phi_{uv}(s,t) \leq E(y^* \mid \theta),
\end{align}
obtained by applying the reparametrization in~\eqref{equ:energy-min} and using the min-sum swap inequality. If there is a reparametrization such that the lower bound is tight and the minimizer $y_u \in \argmin \theta^\phi_u$ in each node is unique, then $y$ is the unique global optimum of~\eqref{equ:energy-min}. To tighten the lower bound we seek to maximize it in $\phi$.

It is known~(\cite{werner2007linear,bogdan2019discretebook}) that this maximization problem is dual to the natural linear programming relaxation of~\eqref{equ:energy-min}. 

The dual problem has the following advantages:
(i) it is constraint-free; (ii) it is composed of a sum of many simple concave terms, each of which is straightforward to optimize.

\paragraph{BCA algorithms} Block-coordinate ascent methods exploit the structure of the dual by iteratively maximizing it \wrt different blocks of variables (subset of coordinates of $\phi$) such that the block maximization can be solved exactly. Formally, let $\phi_F$ be the restriction of $\phi$ to a subset of its 
coordinates $F \subset \{(u,v,s) \mid u \in \V, v\in \nb(u), s \in \Y \}$, BCA algorithms perform the update:
\begin{align}\label{block-max}
\phi_{F} := \arg\max_{\phi_F} D(\phi)
\end{align}
with different blocks $F$ in a static or dynamic order.

\paragraph{Constrained Dual}
For the purpose of this work, it is convenient to work with the constrained dual:
\begin{align}\label{c-dual}
\textstyle
\max _{\phi} D(\phi) \ \  \mbox{s.t.\ \ } \theta^\phi \geq 0.
\end{align}
The equivalence can be shown by constructing for any solution $\phi$ to the unconstrained dual, a correction preserving the objective value and satisfying the constraints~\cite{bogdan2019discretebook}. We will formulate all BCA algorithms in this paper in a way that they maintain the feasibility to the constrained dual.

\begin{figure}[t]
\centering
\setlength{\tabcolsep}{10pt}
\begin{tabular}{cc}
\includegraphics[width=0.22\linewidth]{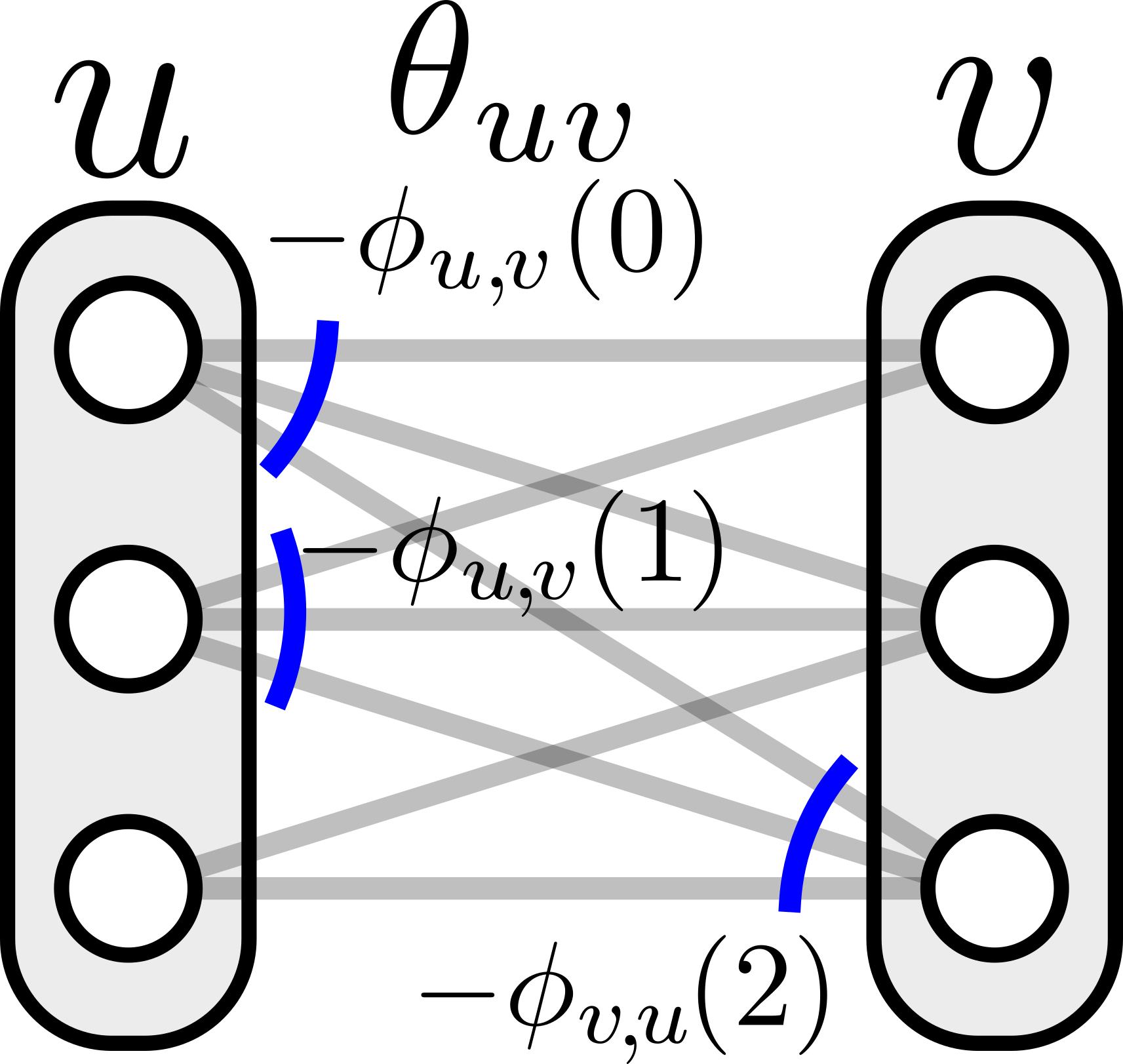}&
\includegraphics[width=0.22\linewidth]{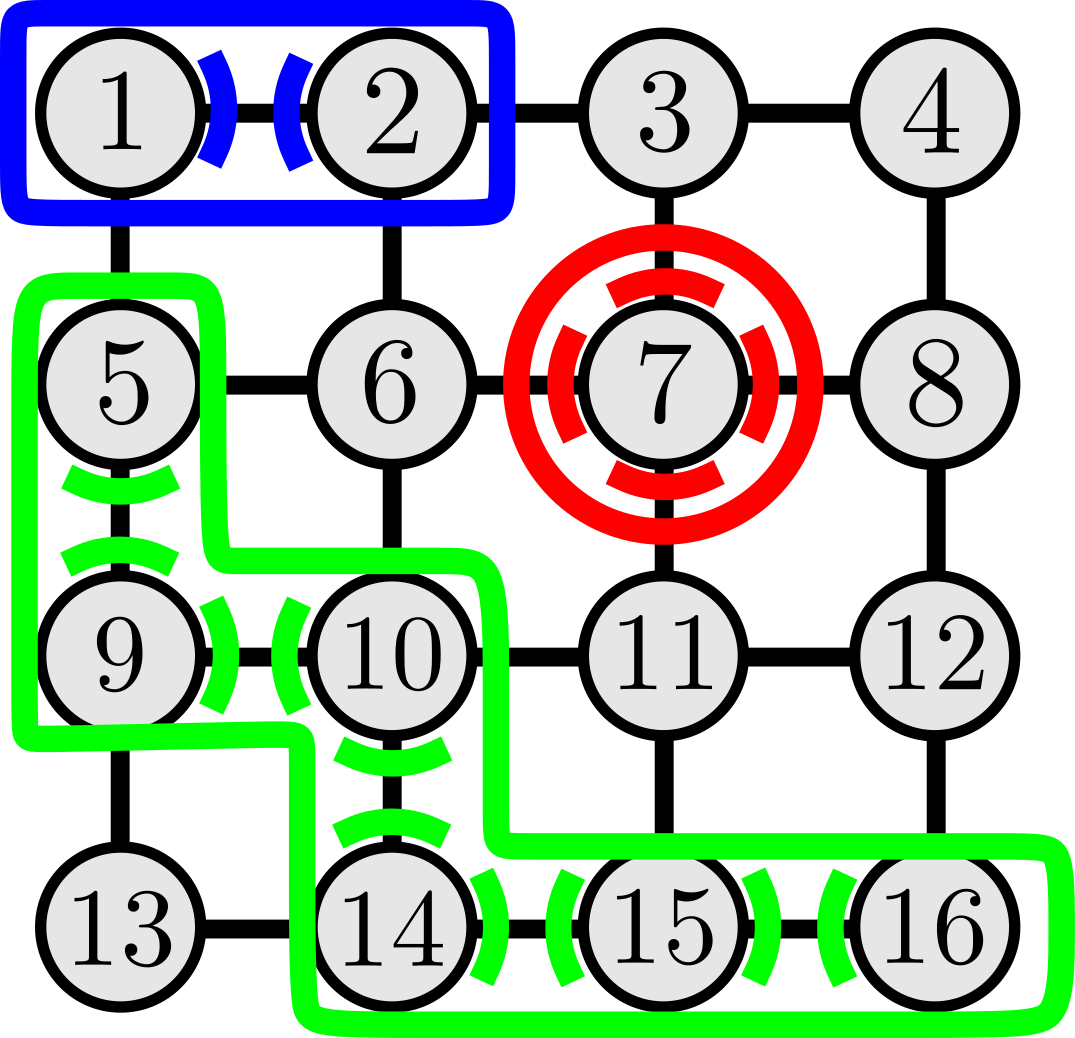}\\
(a) Dual Variables & (b) Dual Blocks
\end{tabular}
\caption{
(a) An edge block and its corresponding reparametrization components in the graphical notation of~\cite{werner2007linear,bogdan2019discretebook}: nodes $u, v$ are shown as grey ovals and circles representing possible labels. The lines connecting the labels represent label pairs $(s,t)$ with associated pairwise costs $\theta_{uv}(s,t)$. For each set of pairwise costs connected to a particular label, there is one reparametrization coordinate $\phi_{u,v}(s)$ shown by a blue arc.
(b) Different variable blocks. Highlighted are block sub-graphs and arcs indicating the variables considered. {\em Red:} node-adjacent block, {\em blue:} edge-block, {\em green} chain block.
\label{fig:1}
}
\end{figure}

\section{Taxonomy of BCA Methods}
We survey a number of BCA methods, listed in~\cref{table:algs}. Many of these methods are derived for different dual objectives and work with different sets of parameters. We reformulate them all as BCA methods on the dual~\eqref{c-dual} and identify the following important design components:

\begin{itemize}[leftmargin=*]
\item Type of blocks used. This has a significant impact on  algorithm efficiency. Larger blocks (such as chains or trees) lead to greater dual improvement, but optimizing over them requires more computations. 

\item Strategy of selecting which block to optimize at every step. A dynamic strategy may be more advantageous for some problems but has additional overhead costs.
\item Type of the update applied. This is not systematically studied in the literature. The maximizer for each block is non-unique but instead it is any point in the optimal facet. One may obtain algorithms with drastically different behaviour, depending on the choice of the maximizer.
\end{itemize}

\subsection{Choice of Variable Block}
BCA algorithms (\cref{table:algs}) exploit the following types of blocks that are tractable to be optimized over:
\begin{itemize}[leftmargin=*]
 \item \textbf{Node-adjacent blocks} $F_u = \{(u,v,s) \mid v\in\nb(u),\ s\in\SY\}$ consist of coordinates of the reparametrization vector that are ``adjacent'' to a node $u$,~\cref{fig:1}(b, red). These blocks are used in \TRWS, \MSD and \CMP algorithms. 
 \item \textbf{Edge blocks} $F_{uv} = \{(u,v,s), (v,u,s) \mid s \in \Y \}$ containing all variables associated with an edge $uv$, see~\cref{fig:1}(b, blue). These are used in \MPLP and \MPLPPP algorithms. 

 \item \textbf{Chains and Trees}
For a sub-graph $(\V', \E') \subset \G$ we select variables associated to all its edges: $F_{\E'} := \cup_{uv\in\E'}F_{uv}$, see~\cref{fig:1}(b, green). To optimize over such blocks, a dynamic programming subroutine is needed. Chain blocks are used \eg in \DMM (rows and columns of a grid graph). The \TRWS algorithm, which we introduced above as a node-adjacent BCA, simultaneously achieves optimality over a large collection of chains. Spanning trees are used in \TBCA variants. We call edge, chain and tree blocks collectively as {\em subgraph blocks}.
\end{itemize}
We will investigate which type of blocks and respective updates are more efficient.

\subsection{Static vs. Dynamic Blocks}

In dynamic \TBCA~\cite{tarlow2011dynamic} the trees are found dynamically by estimating where the dual can be increased the most (so-called local primal-dual gap~\cite{tarlow2011dynamic}), which showed a significant practical speed-up in some applications~\cite{tarlow2011dynamic}.
In other methods, the blocks are fixed in advance: \eg rows and columns for grid graphs in \DMM, single edge blocks in \MPLP, spanning trees, selected greedily to cover the graph, in the static \TBCA.
We will investigate static and dynamic strategies for several update types.
\subsection{Choice of The Local Maximizer}\label{sec:update-types}
With the same blocks one could get very different algorithms depending on how the block maximizer is selected from the polyhedron of possible optimizers, which we refer to as {\em update type}. We can systematize all used update types for node-adjacent blocks and subgraph-based blocks using several elementary operations. We now review them one by one.

\begin{figure}[!t]
\center
\begin{subfigure}[b]{0.5\linewidth}
\center
\includegraphics[height=0.25\linewidth]{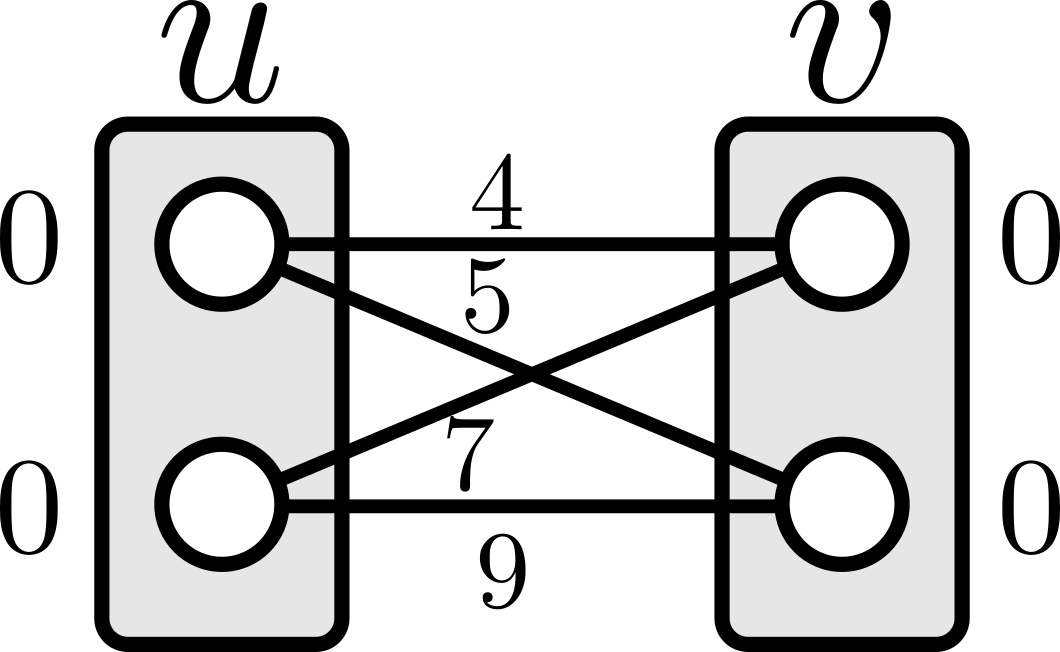}
\caption{Original Edge Block\label{fig:two-block}}
\end{subfigure}%
\begin{subfigure}[b]{0.5\linewidth}
\vspace{0.3cm}
\center
\includegraphics[height=0.25\linewidth]{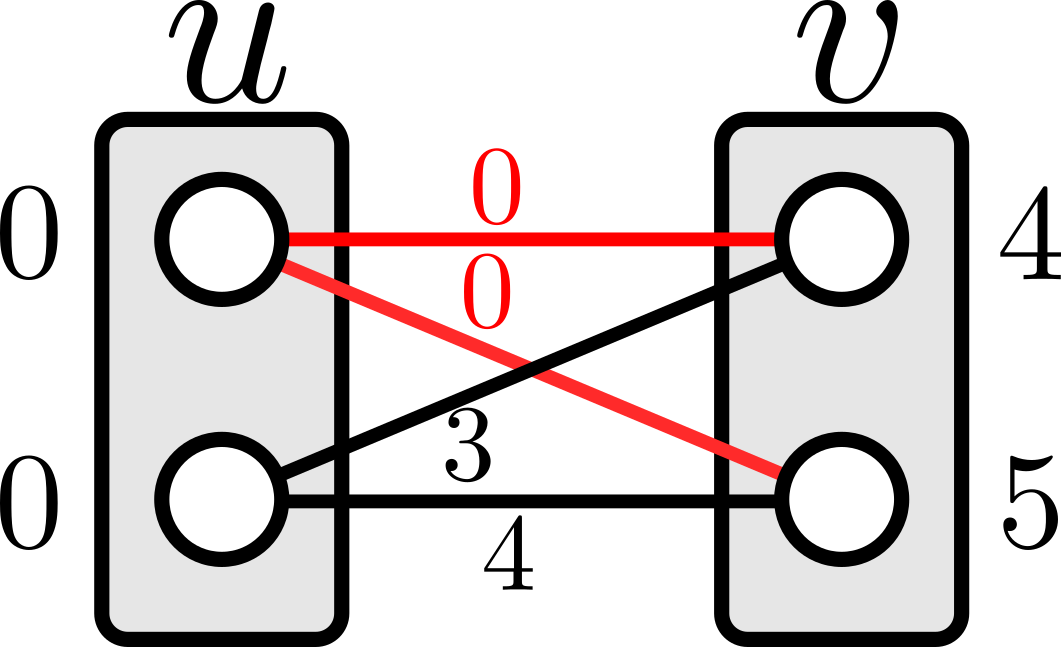}
\caption{Dynamic Programming\label{fig:two-block-dp}}
\end{subfigure}\\
\begin{subfigure}[b]{0.5\linewidth}
\center
\includegraphics[height=0.25\linewidth]{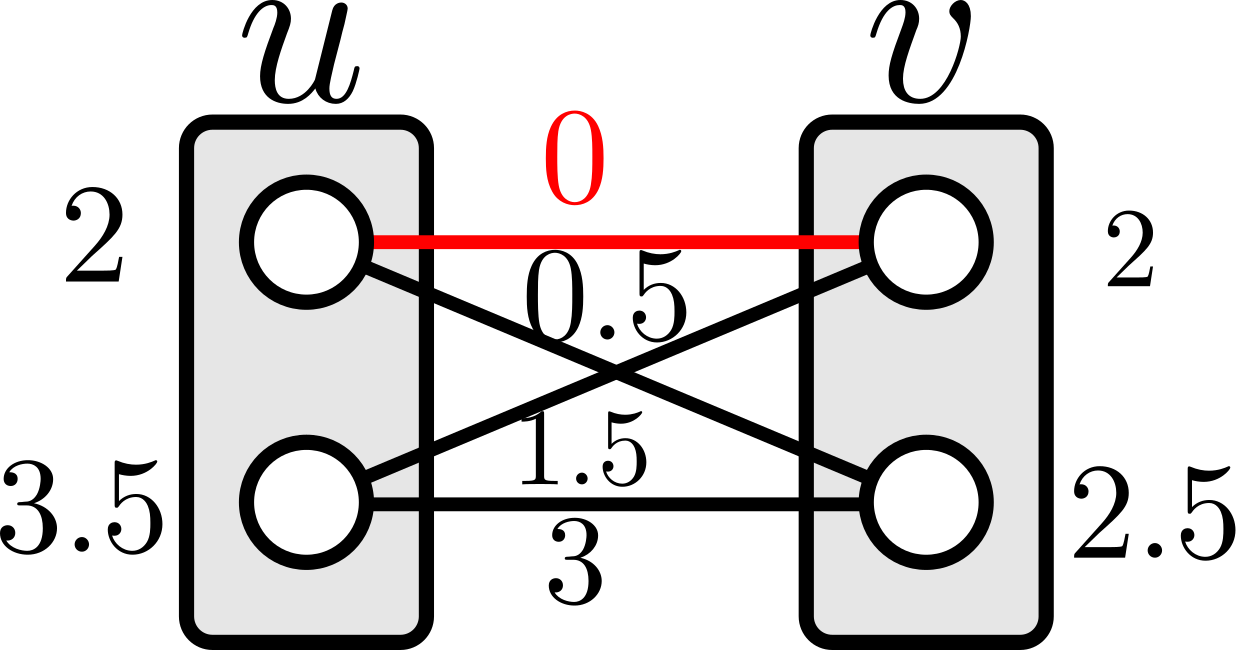}
\caption{\MPLP\label{fig:two-block-mplp}}
\end{subfigure}%
\begin{subfigure}[b]{0.5\linewidth}
\center
\includegraphics[height=0.25\linewidth]{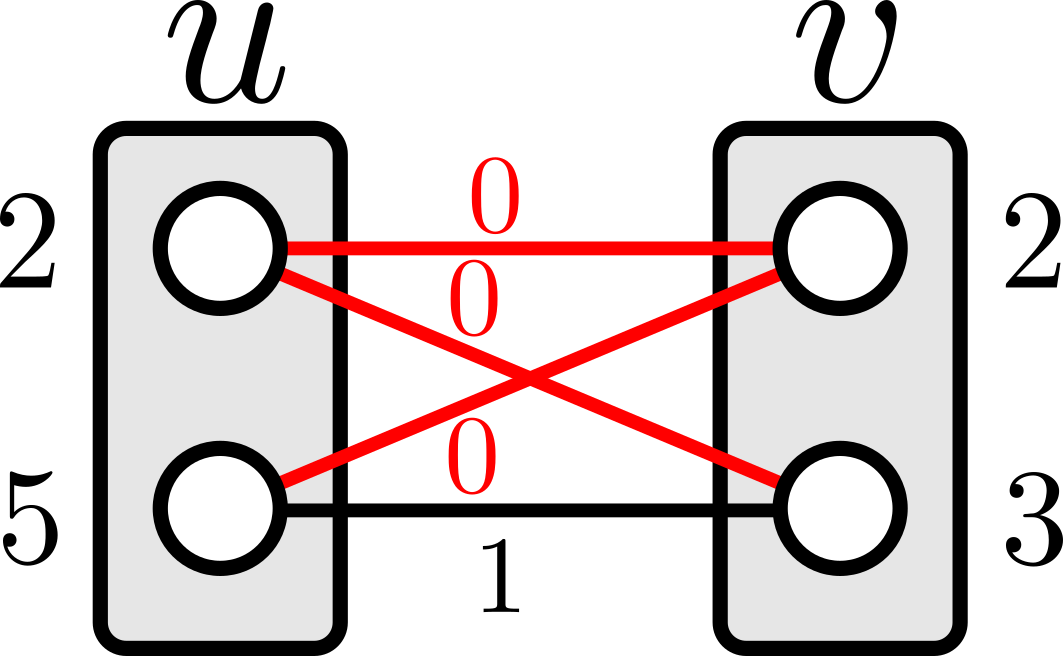}
\caption{{\tt Handshake} ($\HS$)\label{fig:two-block-mplppp}}
\end{subfigure}
\caption{
Example of edge redistribution operations and the difference between non-maximal and maximal minorants; (a) initial edge block with some pairwise costs only; (b) \tDP update; (c) result of \MPLP update; (d) result of {\tt Handshake} update. Observe that node costs created by {\tt Handshake} are strictly bigger than those of \MPLP and more pairwise costs are made zero.  \label{fig:two-block-minorant}}%
\end{figure}

\paragraph{Node-Adjacent Updates}
The update of blocks $F_u$ works in two operations performing {\em aggregation}~\eqref{node-adjacent-up} and {\em distribution}~\eqref{equ:node-adjacent-distribution} for every label $s$ of $u$:
\begin{align}\label{node-adjacent-up}
& \forall v\in\nb(u)\ \phi_{u,v}(s) := \phi_{u,v}(s) - \min_{l\in\SY}\theta^{\phi}_{uv}(s,l)\,,\\
& \forall v\in\nb(u)\ \phi_{u,v}(s) := \phi_{u,v}(s) + w_{u,v} \theta^{\phi}_{u}(s)\,, \label{equ:node-adjacent-distribution}
\end{align}
where coefficients $w_{u,v}$ are non-negative and satisfy $\sum_{v\in\nb(u)}w_{u,v}\le 1$. After the aggregation, the reparametrized costs $\theta^\phi_{uv}$ stay non-negative and label pairs that have zero cost are consistent with the minimizers of the unary reparametrized costs $\theta^\phi_{u}$. This step achieves block maximum~\eqref{block-max}.  The purpose of the distribution step is to redistribute the cost excesses back to the edges of the block (leaving a fraction $1 - \sum_{v\in\nb(u)}w_{u,v}$ at the node $u$) while preserving block optimality. In effect, the neighbouring nodes receive information about good labels for $u$.
The \MSD, \CMP, dynamic programming (\tDP) and \TRWS algorithms are obtained by the respective setting of weights:
\begin{align}
\textstyle
&\mbox{\MSD:\ } \textstyle w_{u,v}{=}\frac{1}{|\nb(u)|}, \
\mbox{\CMP:\ } \textstyle w_{u,v}{=}\frac{1}{|\nb(u)|+1},\\
&\mbox{\tDP:\ } \textstyle w_{u,v}{=}\leftbb v{>}u \rightbb,\ 
\notag \mbox{\TRWS:\ } \textstyle w_{u,v}{=}
\frac{\leftbb v{>}u \rightbb}{\max\{N_{\rm in}(u),\,N_{\rm out}(u)\}},
\end{align}
where $\leftbb \cdot\rightbb$ are Iverson brackets and the other details follow.
The \MSD and \CMP algorithms do not express any preferences in direction (are isotropic) and the order of updating blocks is not as important. %
Updates of \tDP and \TRWS are anisotropic and depend on the order of the vertices. Let us see how \tDP updates work. Consider a chain graph and the chain ordering of nodes. For an inner node $u$ there are two neighbouring nodes: $u\,{-}\,1$ and $u\,{+}\,1$. By choosing $w_{u, u+1} = 1$ and $w_{u, u-1} = 0$, we let all the excess costs  be pushed forward and implement the forward pass of the Viterbi algorithm. \TRWS considers some order of processing of the nodes and applies coefficients $w_{u,v}$ such that for $v<u$ it is zero and for $v>u$ the coefficients are distributed evenly based on the numbers $N_{\rm in}(u)$, $N_{\rm out}(u)$ of incoming and outgoing edges in $u$ \wrt the node order. Note that when there are more incoming edges than outgoing, these weights sum to less than one, \ie some cost excess is left at the node $u$. 
It is clear that the choice of the block update and the order may be crucial in BCA methods.

In contrast to node-adjacent blocks, \textbf{subgraph blocks} overlap only in the nodes of the graph. Therefore, different redistribution strategies have been proposed in order to make the excess costs visible in all nodes of the processed block.

\paragraph{Edge Updates}
\begin{subequations}\label{MPLP-agg}
\MPLP and \MPLPPP methods consider edge blocks. \MPLP performs the following symmetric update $\phi:=\MP_{u,v}(\theta, \phi)$:
\begin{align}
&\textstyle \forall s{\in}\SY, \quad \phi'_{u,v}(s) := \phi_{u,v}(s) + \theta^\phi_u(s), \label{MPLP-agg-a}\\ 
&\textstyle \forall t{\in}\SY, \quad \phi'_{v,u}(t):=\phi_{v,u}(t) + \theta^\phi_v(t);  \label{MPLP-agg-b}
\end{align}
\end{subequations}%
\vskip-\baselineskip
\begin{subequations}\label{MPLP-dist}
\begin{align}
&\textstyle \forall s{\in}\SY, \ \phi_{u,v}(s):=\phi'_{u,v}(s){-}\tfrac{1}{2}\min_{t \in\SY}\theta^{\phi'}_{uv}(s,t),\label{MPLP-dist-a}\\ 
&\textstyle \forall t{\in}\SY, \ \phi_{v,u}(t):=\phi'_{v,u}(t){-}\tfrac{1}{2}\min_{s\in\SY}\theta^{\phi'}_{uv}(s,t). \label{MPLP-dist-b} 
\end{align}
\end{subequations}
The {\em aggregation} step~\eqref{MPLP-agg} achieves that the costs in the nodes $u$, $v$ become aggregated in the edge $uv$. Reparametrized costs $\theta^{\phi'}_{u}$, $\theta^{\phi'}_{v}$ become zero and $\theta^{\phi'}_{uv}(s,t) = \theta^\phi_{uv}(s,t) + \theta^\phi_{u}(s) + \theta^\phi_{v}(t)$ does not depend on the initial reparametrization components $\phi_{uv}, \phi_{vu}$. At this point the maximum over the edge block is found. 

The {\em distribution} step~\eqref{MPLP-dist-a}-\eqref{MPLP-dist-b} divides the aggregated cost in two halves and pushes the excesses from each half back to two nodes $u$, $v$, to make the preferred solution for the edge visible in the nodes. See~\cref{fig:two-block-mplp}.

The {\tt Handshake} ($\HS$) update $\phi:=\HS_{u,v}(\theta, \phi)$ is used in \MPLPPP and \DMM.
It differs in the distribution step. Let $\phi_{u,v}(s)$ be computed as in \MPLP and $\phi_{v,u}(t)$ set to an arbitrary value. The {\tt Handshake} update additionally performs:
\begin{subequations}
\begin{align}
\forall t\in\SY, & \quad \phi_{v,u}(t) := \phi_{v,u}(t) - \min_{s \in\SY}\theta^{\phi}_{uv}(s,t),\label{equ:handshake-1}\\ 
\forall s\in\SY, & \quad \phi_{u,v}(s) := \phi_{u,v}(s) - \min_{t \in\SY}\theta^{\phi}_{uv}(s,t).\label{equ:handshake-2}
\end{align}
\end{subequations}
This step pushes the still remaining cost excess from the edge to the nodes as illustrated in~\cref{fig:two-block-mplppp}.
It leads to a strictly better improvement of the dual objective after the pass over all blocks and performs considerably better in experiments~\cite{Tourani_2018_ECCV}. The step~\eqref{equ:handshake-1} does not depend on the value of $\phi_{v,u}(t)$, which can be seen by moving $\phi_{v,u}(t)$ under the $\min$ and expanding the reparametrization. Therefore step~\eqref{MPLP-dist-b} may be omitted when computing $\HS$. 

\paragraph{Chain / Tree Updates}
The optimality over an edge, chain or a tree can be achieved by applying the following {\em dynamic programming} update $\phi:=\DP_{u,v}(\theta, \phi)$ (in the order of the chain or from leaves to the root of a tree):
\begin{subequations}
\begin{align}
\forall s \in \SY, & \quad \phi_{u,v}(s):= \phi_{u,v}(s) + \theta^{\phi}_{u}(s), \label{DP-1}
\end{align}
\begin{align}
\forall t \in \SY, & \quad \phi_{v,u}(t):= \phi_{v,u}(t) -\min_{s\in \SY}[\theta^\phi_{u,v}(s,t)]\,. \label{DP-2}
\end{align}
\end{subequations}
The step~\eqref{DP-1} aggregates the cost excess from node $u$ to the edge $uv$ and the step~\eqref{DP-2} pushes the cost excess to note $v$. Observe that it can be written in the form of a node-adjacent update~\eqref{node-adjacent-up}-\eqref{equ:node-adjacent-distribution} by grouping the push step~\eqref{DP-2} into $v$ with the aggregation step~\eqref{node-adjacent-up} at $v$ when processing the next edge $vw$.

\TBCA algorithm uses \tDP to achieve optimality over a tree and then performs a pass in the reverse order, redistributing the costs with the following \texttt{rDP} update.

{\em Redistribution DP update} $\phi:=\rDP_{u,v}(\theta, \phi)$
\begin{subequations}
\begin{align}
\forall s \in \SY, & \quad \phi_{u,v}(s):= \phi_{u,v}(s) + r \theta^{\phi}_{u}(s), \label{rDP-1} \\
\forall t \in \SY, & \quad \phi_{v,u}(t):= \phi_{v,u}(t) -\min_{s\in \SY}[\theta^\phi_{u,v}(s,t)], \label{rDP-2}
\end{align}
\end{subequations}
where $0 \le r \le 1$ is a constant similar to the weights in the node-adjacent updates. The fraction $r$ of cost excess is pushed forward to $v$ and the fraction $1-r$ is left in the node $u$. \TBCA detailed in~\cref{alg:tbca,fig:tbca-message-passing} redistributes cost excesses based on the size of the tree branch remaining ahead. \TBCA was originally proposed for the dual decomposition with trees~\cite{sontag2009tree} and works with its Lagrange multipliers.

\DMM  works with chain subproblems and performs the redistribution hierarchically as explained in~\cref{alg:dmm,fig:hm-message-passing}. 
It was also originally proposed for the dual decomposition with chains~\cite{Discrete-Continuous-16}.
One advantage of this method is that when the chain contains an edge with zero pairwise costs (no interactions), the processing becomes equivalent to redistribution in two chains independently. In contrast, the \TBCA method would be confused in its estimate of the size of the subtree to push the excess to.

\begin{figure}[t]
\centering
\includegraphics[width=0.5\linewidth]{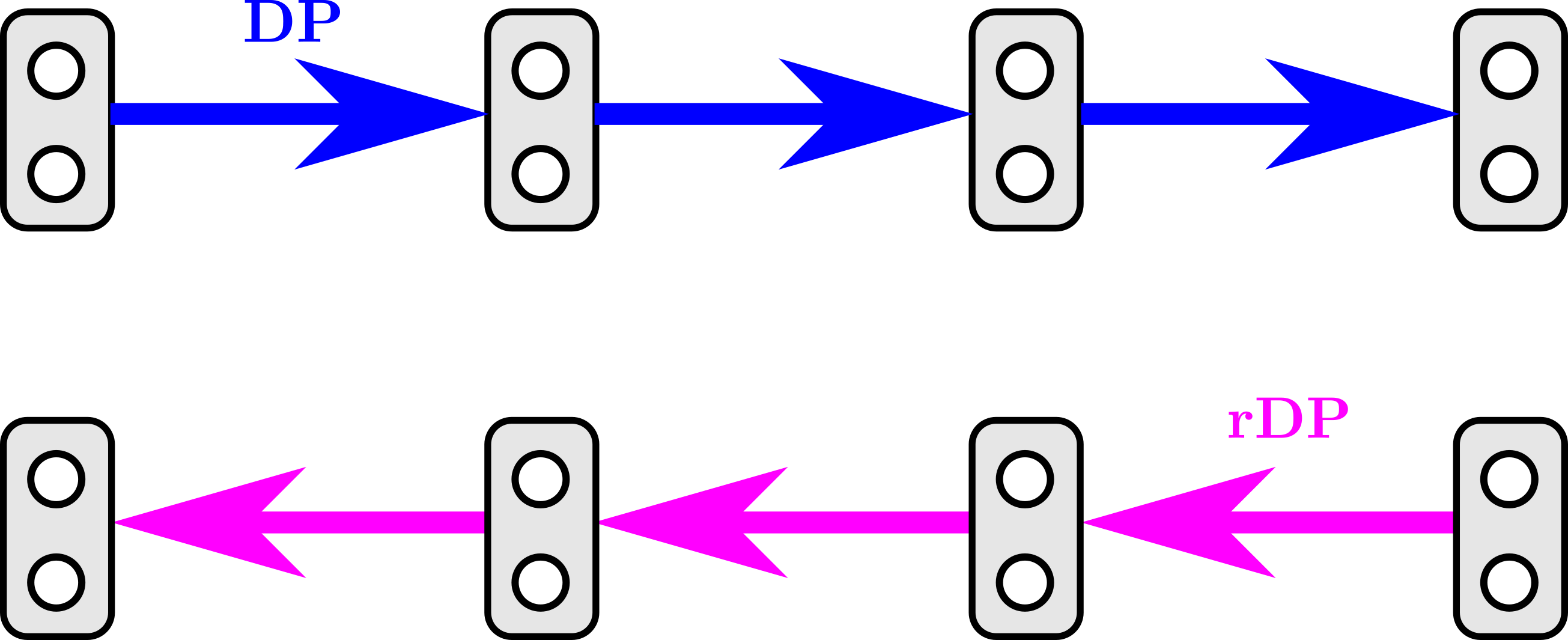}
\caption{\label{fig:tbca-message-passing}
Tree-BCA computation. Blue arrows represent the DP update of the edge in that direction. Pink arrows represent the rDP operations. The \TBCA update starts off by collecting costs at the end of the chain (right) and then redistributes with rDP with $r = (n-i)/n$ in the $i$'th backward step.}
\vskip-0.5\baselineskip
\begin{algorithm}[H]
\caption{Tree-BCA Update on a Chain}
\begin{algorithmic}[1]
\Require $\phi$ - starting reparametrization; 
$(\SV', \SE')$ - chain subgraph arranged as $\SV'=\{1,\dots,n\}$ and $\SE'=\{\{i,i+1\}\colon i=1,\dots,n-1\}$.
\For{$i=1,\dots,\text{n}-1$} \Comment{Collects the costs at the chain end with dynamic programming}
\State $(\phi_{i,i+1},\phi_{i+1,i}):=\DP_{i,i+1}(\theta, \phi)$ \label{alg:line:collect-tbca}
\EndFor 
\For{$i=n,\dots,2$} \Comment{Distributes the costs in reverse order to the nodes}
\State $(\phi_{i-1,i},\phi_{i,i-1}):=\rDP_{i+1,i}(\theta, \phi)$  with $r=\frac{n-i}{n}$ \label{alg:line:distribute-tbca}
\EndFor
\end{algorithmic}
\label{alg:tbca}
\end{algorithm}
\vskip-\baselineskip
\end{figure}

\section{Analysis}
Subgraph-based blocks usually overlap over the nodes only (horizontal / vertical chains) or have a small overlap over the edges (spanning trees). Consider two blocks that overlap over nodes only. The representation of the information (costs of different solutions) which is available to one block about the other is limited to the reparametrized node costs $\theta^\phi_{u}$ for all shared nodes $u$. We identify this reparametrized unary potentials with modular minorants~\cite{Discrete-Continuous-16}, having clear analogies with minorants/majorants in pseudo-Boolean optimization~\cite{BorosHammer02}.
\subsection{Modular Minorants} 

A function $f \colon \SY^n \to \BR$ of $n$ discrete variables is called {\em modular}, if it can be represented as a sum of $n$ functions of one variable: $f(y)=\sum_{i=1}^n f_i(y_i)$, $f_i\colon \SY\to\BR$. 
The function $f(y)=\sum_{u\in\SV'}\theta^{\phi}_u(y_u)$ is modular for any subset of nodes $\SV' \subseteq\SV$ and any reparametrization~$\phi$.
\begin{definition}\label{def:modular-minorant}
 A modular function $g$ is called a {\em (tight) minorant} of $f\colon \SY^n\to\BR$, if (i) $g(y)\leq f(y)$ asnd (ii) $\min_{y\in\SY^n}f(y) = \min_{y\in\SY^n}g(y)$.
\end{definition}

For the rest of this section we will assume that $\G' = (\V', \E')$ is a subgraph defining a block of variables optimized at one step of a BCA algorithm and $E_{\SG'}$ is the restriction of energy $E$ to graph $\G'$ with the reparametrized costs~$\theta^\phi$. A reparametrization $\phi$ is called {\em dual optimal on $\G'$}, if it is block-optimal in the sense of~\eqref{block-max} \wrt block $F_{\E'}$.
Minorants and dual optimal reparametrizations are closely related: 
\begin{restatable}{theorem}{Tminorant}\label{prop:minorant-property}
Let $\G'$ be a tree and $\sum_{uv\in\SE'}\min_{s,t}\theta^\phi_{uv}(s, t) =0$.
The function $g(y)=\sum_{u\in\SV'}\theta^{\phi}_{u}(y_u)$
is a minorant for the energy $E_{\SG'}(y)$ \emph{if and only if} $\phi$ is dual optimal on~$\G'$.
\end{restatable}

To put it differently, if $\SG'$ defines a sub-graph block for a block-coordinate ascent method, then choosing amongst block optimal reparametrizations $\phi$ is equivalent, up to a constant, to choosing a modular minorant for the energy $E_{\SG'}$. 

Observe that, for a sub-graph block $(\V', \E')$, there are $2 |\E'| |\Y|$ reparametrization variables but only $|\V'||\Y|$ coordinates are needed to define a minorant. The minorant naturally captures the degrees of freedom that are important for subgraph-based BCA methods. 

Minorants can be partially ordered with respect to how tightly they approximate the function. For two minorants $g,g'$ we write $g' \ge g$ if $g'(y)\ge g(y)$ for all $y\in\SY^{\V'}$. Since our minorants are modular, the condition is equivalent to component-wise inequality  $g'_u(s)\ge g_u(s)$ $\forall u\in \V', \forall s\in \Y$. 
%
The greater the minorant, the tighter it approximates the function. Hence, of interest are maximal minorants:
\begin{definition}[\cite{Discrete-Continuous-16}] A minorant $g$ is {\em maximal}, if there is no other minorant $g'\ge g$ such that $g'(y) > g(y)$ for some $y$.
 \end{definition}
For the best performance of a BCA method, it makes sense to select a maximal minorant and not just any minorant. To actually apply this idea to BCA methods, we show how the maximality property of a minorant translates back to reparametrizations:

\begin{restatable}{theorem}{Tmaxminorant}\label{thm:max-minorant-property}
Let $\G'$ be a tree and reparametrization $\phi$ be dual optimal on $\SG'$. The function $g(y)=\sum_{u\in\SV'}\theta^{\phi}_{u}(y_u)$ is a maximal minorant if and only if $\forall uv\in\SE'$ and $\forall s,t\in\SY$:
\begin{equation}\label{equ:max-minorant-condition}
 \min\nolimits_{s'\in\SY}\theta^{\phi}_{uv}(s',t) = \min\nolimits_{t'\in\SY}\theta^{\phi}_{uv}(s,t')=0\,.
\end{equation}
\end{restatable}
\vskip-0.5\baselineskip
With these results we can now draw conclusions about algorithms updating subgraph blocks.

All BCA methods considered, as they achieve block optimality, construct minorants. However, many of them are not maximal. Minorants constructed by \MPLP and \TBCA are non-maximal. The change introduced in \MPLPPP achieves maximality as illustrated in~\cref{fig:two-block-mplppp}. This minor change brings more than an order of magnitude speed-up to the algorithm in some problem instances~\cite{Tourani_2018_ECCV}. The correction can be extended to \TBCA, also leading to improvements without any further changes, \cref{sec:TRWSvs}.

On the other side, the connection we established allows to interpret \DMM as a BCA method working on the dual~\eqref{c-dual} and identify its reparametrization form as presented. 

\begin{figure}[t]
\centering
\includegraphics[width=0.80\linewidth]{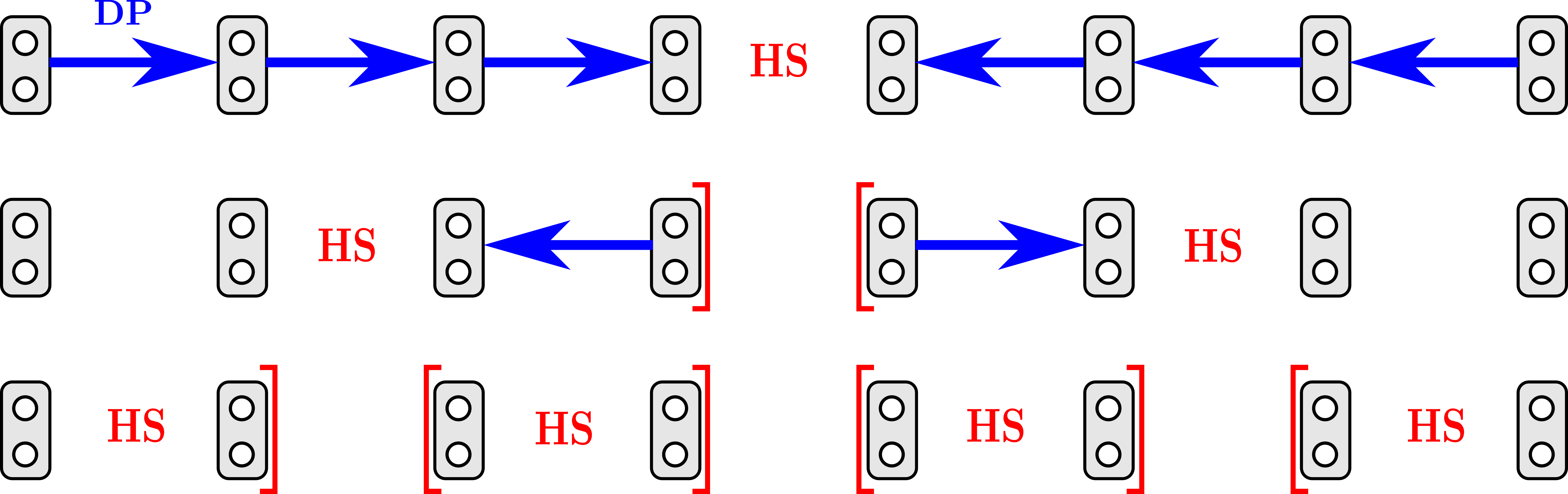}
\caption{\label{fig:hm-message-passing}
Hierarchical Minorant computation. The top level of hierarchy computes DP updates (blue arrows) towards the central edge. At the central edge the {\tt Handshake} update (red HS) is applied. Then the same is applied recurrently to the two formed sub-chains. The messages that have been already computed are kept from preceding levels. When recurrence completes, the HS operation has been applied to every edge, resulting in a \emph{maximal-minorant}.}
\end{figure}
\begin{algorithm}[t]
\caption{Chain\,Hierarchical\,Minorant\,({\HM})\,Update}
\begin{algorithmic}[1]
\Require $\phi$ - starting reparametrization; chain sub-graph $(\V', \E')$ arranged as $\SV=\{1,\dots,n\}$ and $\SE=\{\{i,i+1\}\colon i=1,\dots,n-1\}$.
\State $i_{L}:=\floor{\frac{n}{2}}$; $i_R:=i_L+1$  \Comment {Compute the left and right mid-points of the chain}
\For{$i=1,\dots,i_R - 1$} \Comment{Push costs from chain start to $i_L$}
\State $(\phi_{i,i+1},\phi_{i+1,i}):=\DP_{i,i+1}(\theta, \phi)$
\EndFor
\For{$i=n,\dots,i_R+1$} \Comment{Push costs from the chain end to $i_R$}
\State $(\phi_{i-1,i},\phi_{i,i-1}):=\DP_{i,i-1}(\theta, \phi)$
\EndFor
\State $(\phi_{i_L,i_R},\phi_{i_R,i_L}):=\HS_{i_L,i_R}(\theta, \phi)$ \Comment{Handshake update on the middle edge}
\State Call \HM on two sub-chains $[1\dots i_L]$ and $[i_R \dots n]$ if not empty
\end{algorithmic}
\label{alg:dmm}
\end{algorithm}
\begin{figure*}[!ht]
\centering
\setlength{\tabcolsep}{5pt}
\setlength{\lineskip}{0pt}
\renewcommand*{\arraystretch}{0}
\begin{tabular}{ccc}
{\scriptsize \bf Sparse} & {\scriptsize \bf Denser} & {\scriptsize \bf Complete}\\[0pt]
\begin{tabular}{c}
\includegraphics[width=0.26\linewidth]{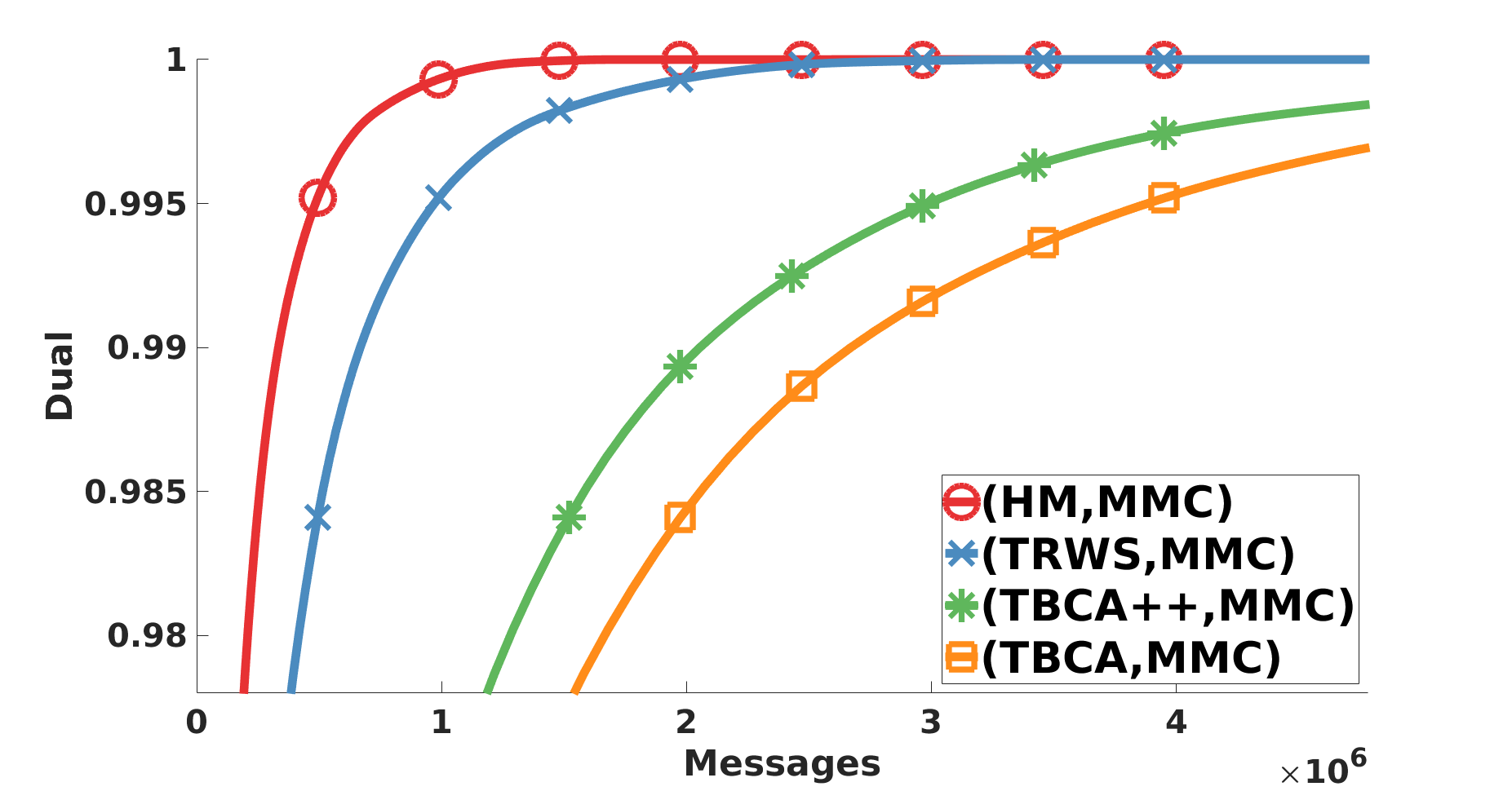}%
\end{tabular}&%
\begin{tabular}{c}
\includegraphics[width=0.26\linewidth]{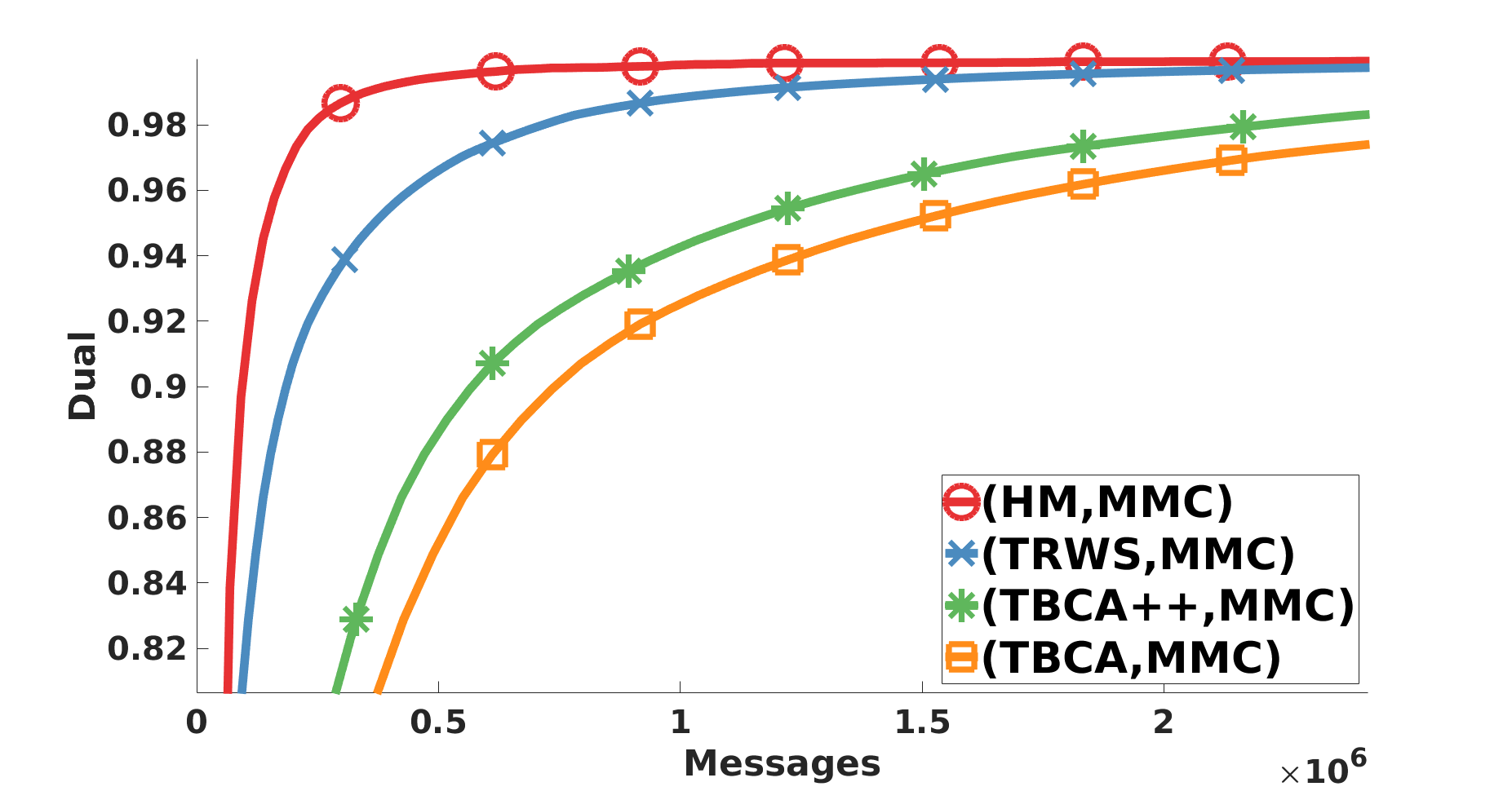}%
\end{tabular}&%
\begin{tabular}{c}
\includegraphics[width=0.26\linewidth]{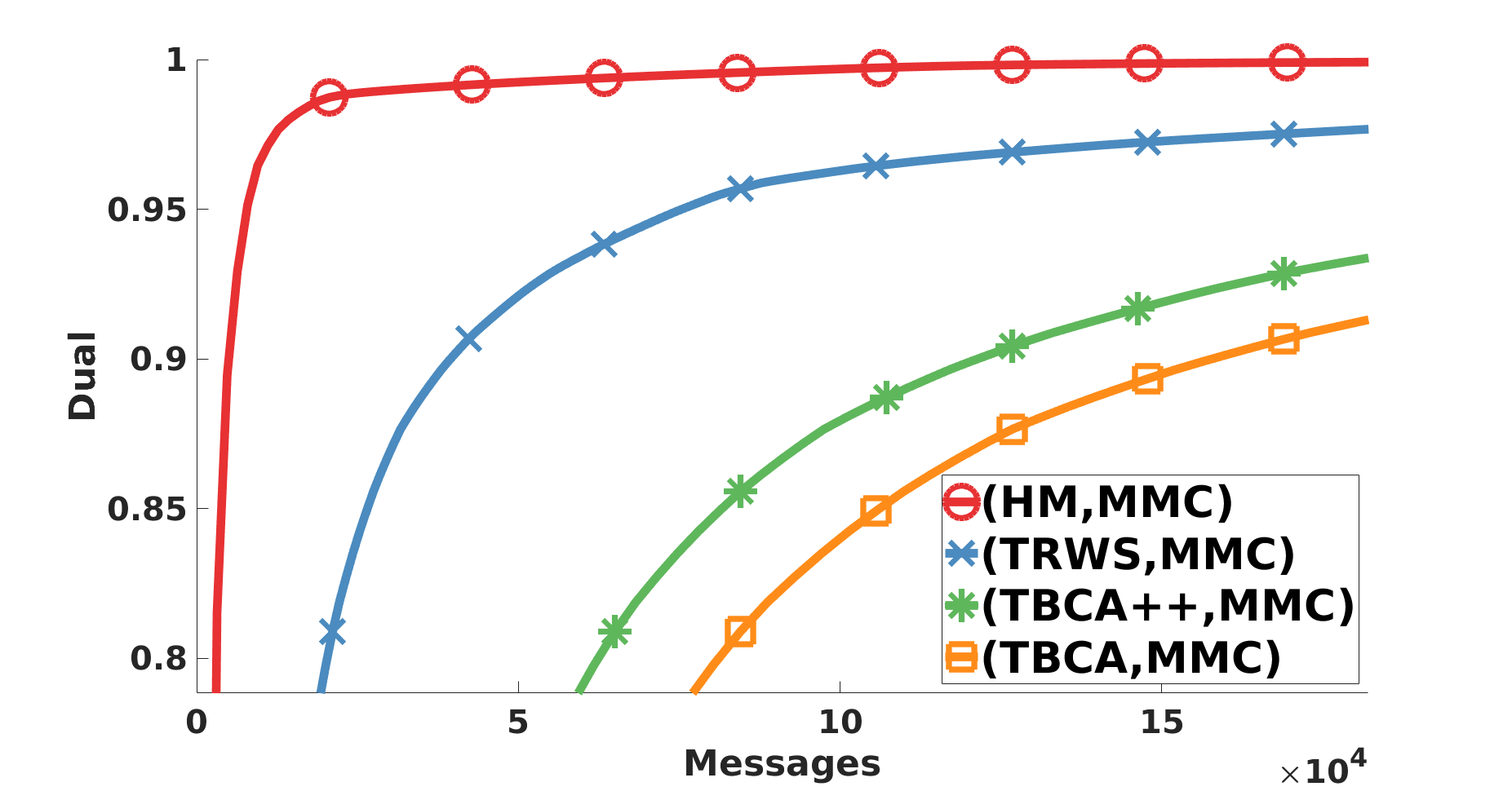}
\end{tabular}
\end{tabular}
\caption{Comparison between \TRWS efficiently optimizing all monotonic chains with subgraph-based updates on a covering subset of monotonic chains. See description of datasets in~\cref{tab:datasets}.
\label{fig:TRWS-vs-subgraph}}
\end{figure*}
\section{Synthesis}	\label{sec:SPAM}
Based on the above analysis and the experimental comparison of individual components of BCA methods, we synthesize the following BCA algorithm that appears to perform universally better in terms of achieved dual objective value versus time on a corpus of diverse problems. Here are the design choices that we made:
\begin{itemize}[leftmargin=*]
\item We utilize chain blocks and the hierarchical minorant ({\tt HM}, Algorithm~\ref{alg:dmm} and Fig.~\ref{fig:hm-message-passing}) updates. These updates are the most expensive ones, but the maximality property and better redistribution of the excess costs pays off in practice.
\item We observed that with the hierarchical minorant updates, selecting chain blocks dynamically does not give an improvement over a static set of chains, unlike in~\cite{tarlow2011dynamic}.
\item We select chains automatically for a given graph by a new heuristic. This heuristic behaves favourably in both sparse regular graphs as well as dense graphs. This automatic choice allows the method to achieve a uniformly good performance over problems with different graph structure and connectivity.

\end{itemize}
Next we present specifically designed experiments that led to these choices.
\begin{figure}[!t]
\centering
\includegraphics[width=0.6\linewidth]{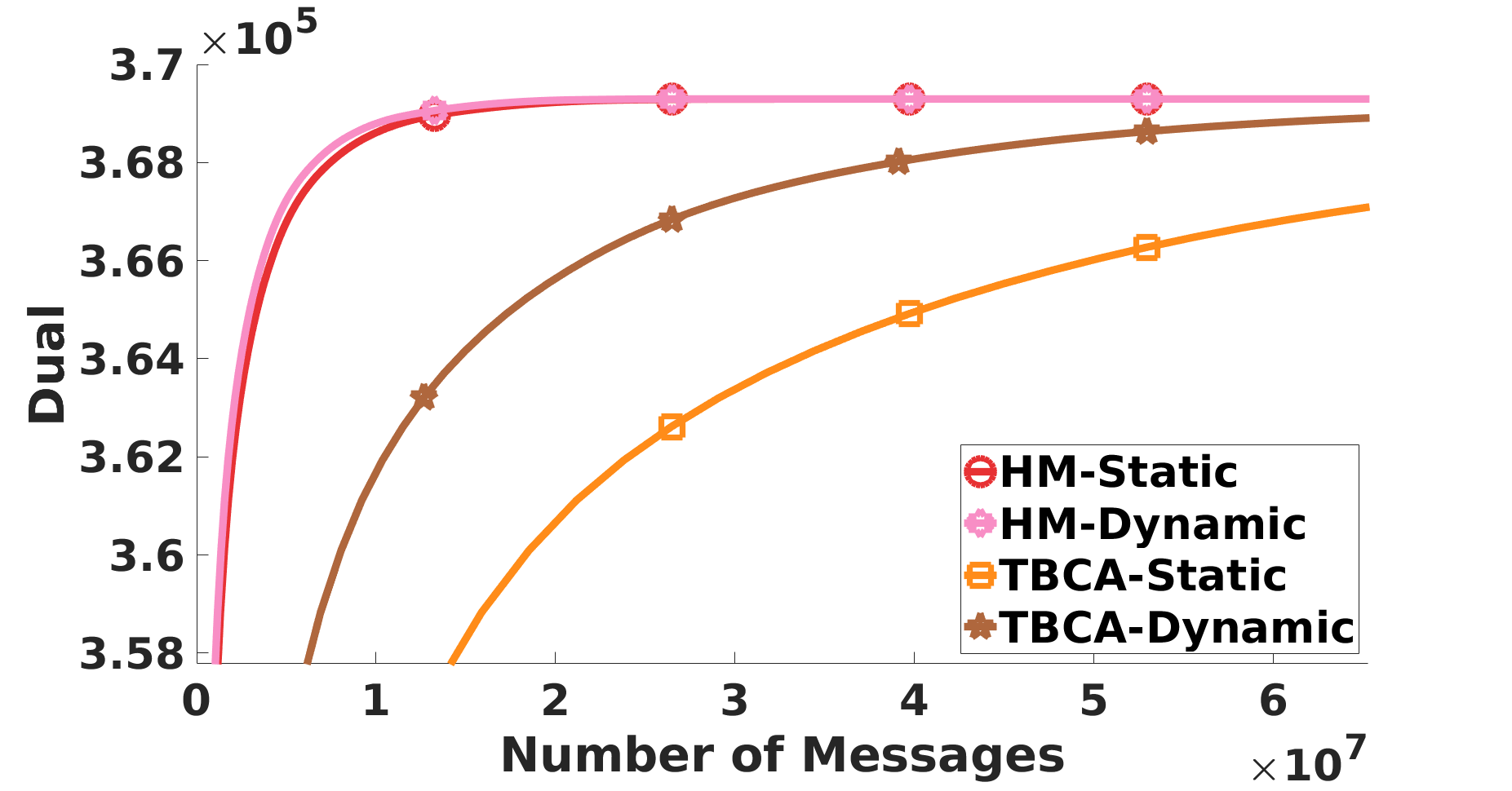}
\caption{Comparison of static and dynamic spanning trees on tsukuba from the stereo dataset. Both the dual and the messages are not normalized.}
\label{fig:local-pd-gap}
\end{figure}
\begin{figure}
\centering
\includegraphics[scale=0.07]{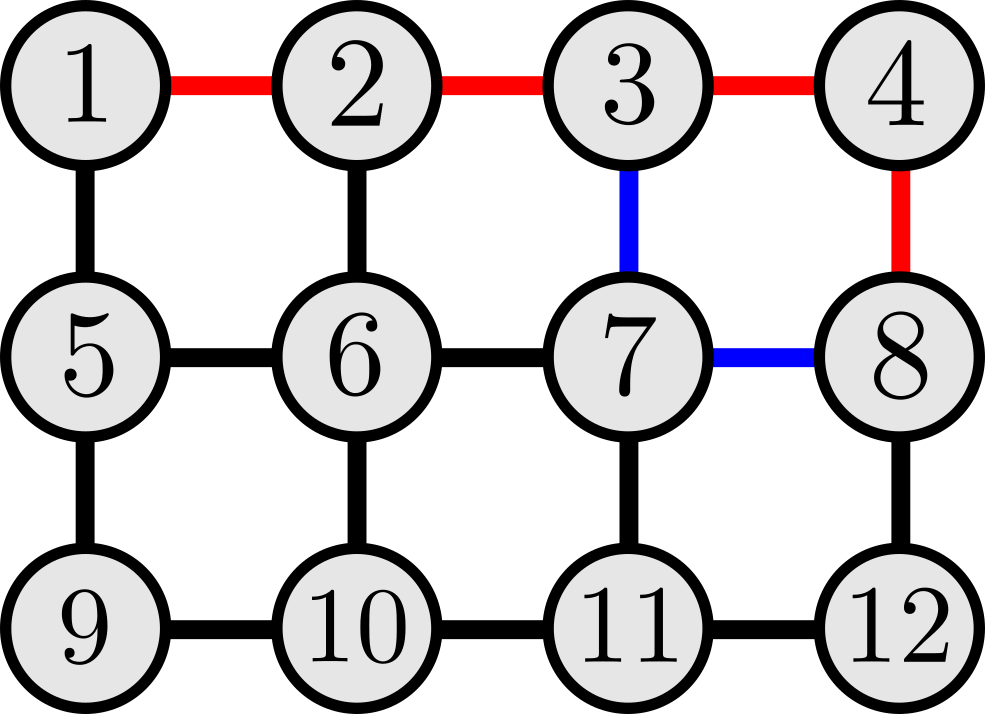}
\caption{Strictly shortest paths. In this example the chain 1-2-3-4-8 and the chain 1-2-3-7-8 are both shortest paths to node 8 and therefore none of them is the strictly shortest path. The chain 1-2-3-4 is the unique (and hence strict) shortest path from 1 to 4.%
\label{fig:strictlyShortestPath}}
\end{figure}

%
\begin{figure*}[!ht]
\centering
\setlength{\tabcolsep}{5pt}
\setlength{\lineskip}{0pt}
\renewcommand*{\arraystretch}{0}
\begin{tabular}{ccc}
{\scriptsize \bf Sparse} & {\scriptsize \bf Denser} & {\scriptsize \bf Complete}\\[0pt]
\begin{tabular}{c}
\includegraphics[width=0.25\linewidth]{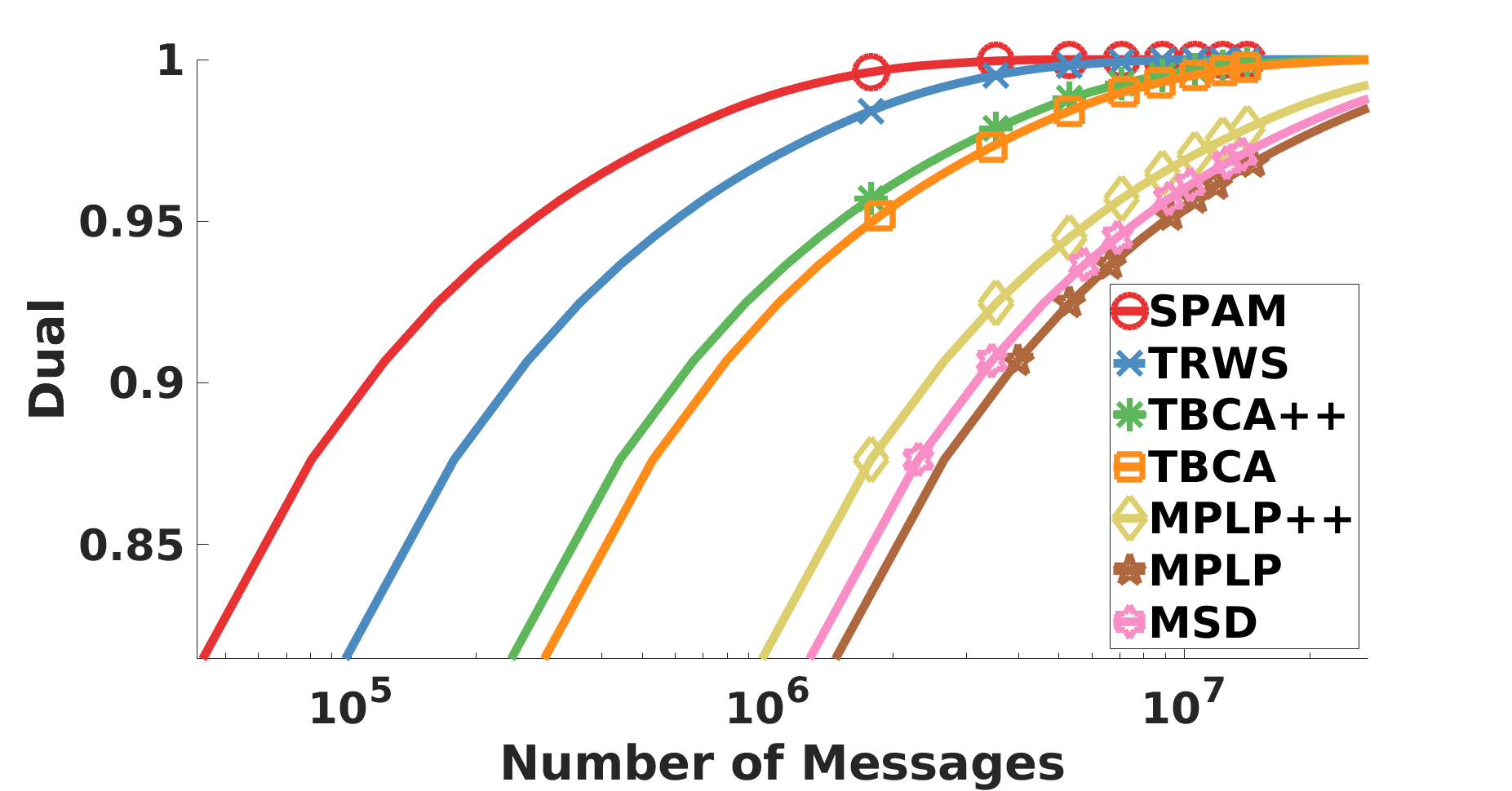}%
\end{tabular}&%
\begin{tabular}{c}
\includegraphics[width=0.25\linewidth]{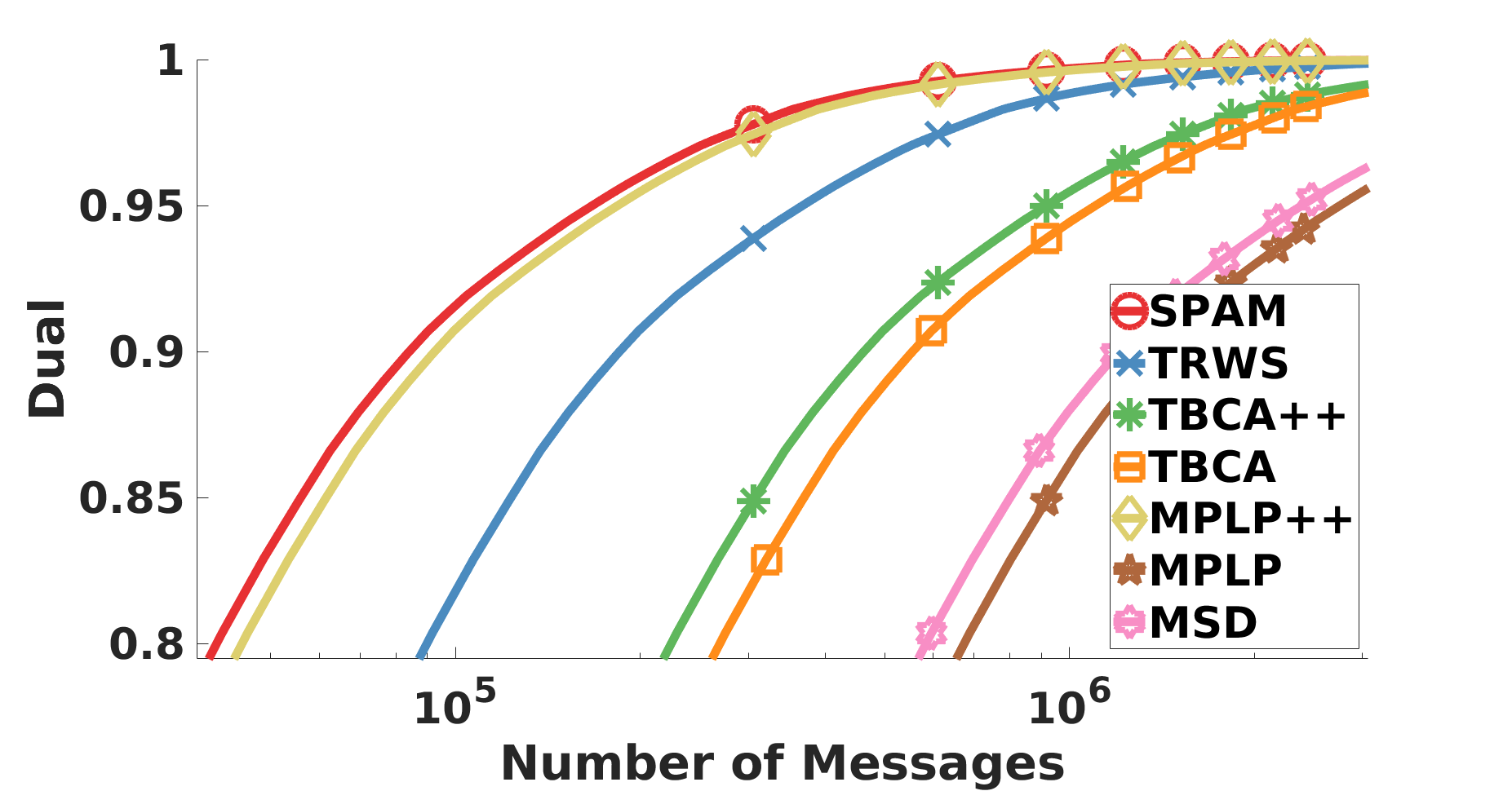}%
\end{tabular}&%
\begin{tabular}{c}
\includegraphics[width=0.25\linewidth]{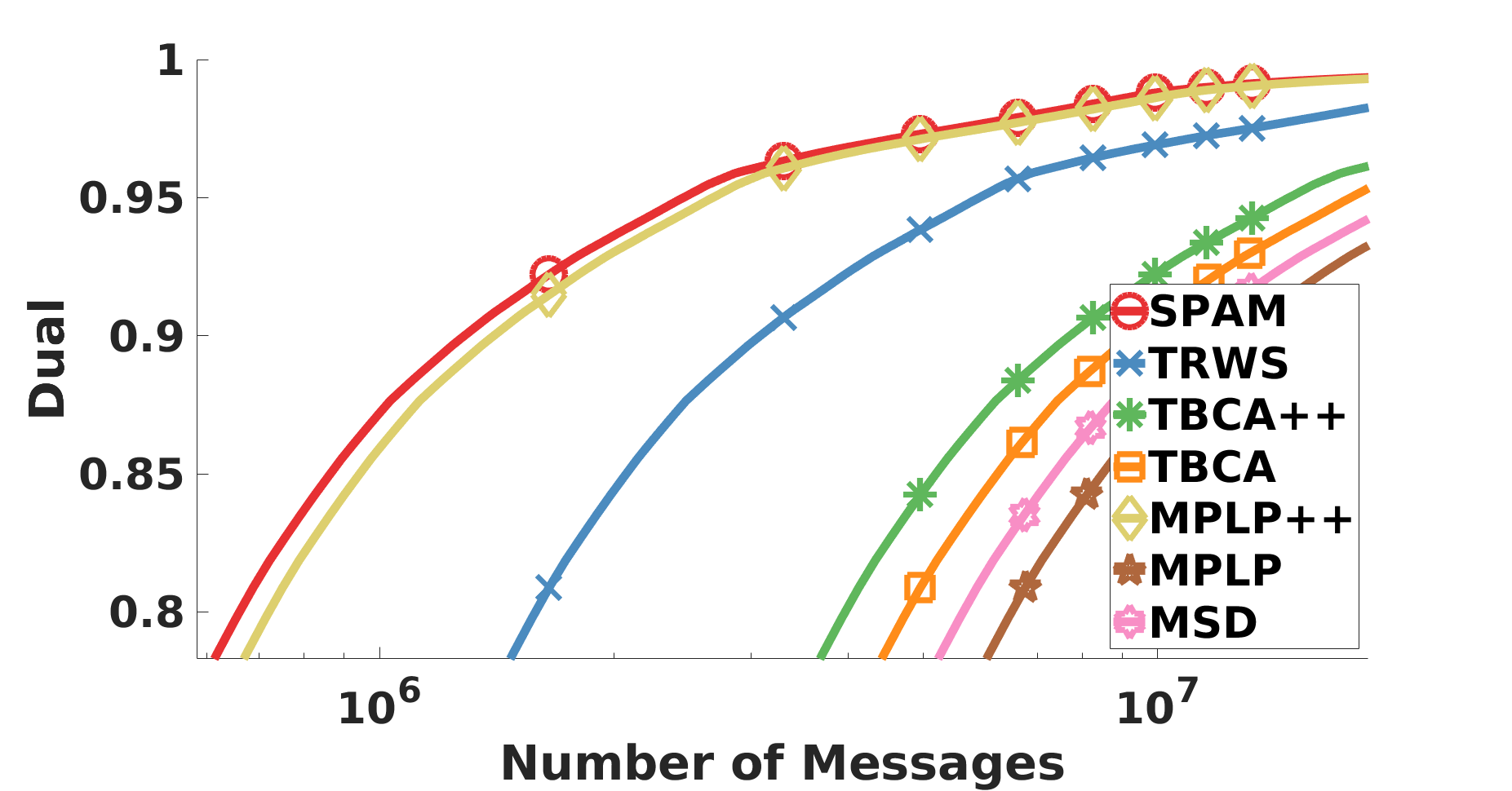}
\end{tabular}
\end{tabular}
\caption{\label{fig:dualComplete}%
Algorithm comparison for sparse, denser and complete graphs following the experimental setup in~\cref{sec:exp-setup}.
The minor difference between \MPLPPP and \SPAM for complete graphs is explained by different order of computations.
The corresponding runtime plots look qualitatively the same and are provided in the appendix.} 

\centering
\vskip\baselineskip
\setlength{\tabcolsep}{5pt}
\renewcommand*{\arraystretch}{1.2}
\begin{tabular}{|p{0.3\linewidth}|p{0.34\linewidth}|p{0.3\linewidth}|}
\hline
{\bf Sparse:}
 \small Problems on 4-connected grid graphs (less than $1\%$ connectivity): instances of 
 \texttt{stereo} ($3$ models with truncated linear pairwise costs and $16$, $20$ and $60$ labels) and \texttt{mrf-inpainting} ($2$ models with truncated quadratic pairwise costs and $21838$ and $65536$ nodes) from the Middlebury MRF benchmark~\cite{szeliski2008comparative}.
& {\bf Denser:} 
\small Problems with connectivity in between grids and complete graphs: \texttt{worms}~\cite{kainmueller2014active} ($30$ instances coming from the field of bio-imaging, $558$ nodes with $20-65$ labels each, around $10\%$ connectivity);
\texttt{denser-stereo} (adds additional longer-range pairwise smoothness interactions to the \texttt{stereo} dataset, for each of 3 instances we create 4 denser variants of increased connectivity (20\%, 30\%, 40\%, 50\%)).
&
{\bf Complete:}
\small Problems with fully connected graphs ($100\%$ connectivity):
 \texttt{protein-folding} model~\cite{dataset-proteinfolding} instances of OpenGM benchmark~\cite{kappes-2015-ijcv} ($11$ instances with $33-40$ nodes and up to $503$ labels per node);
\texttt{pose} $6$D object pose estimation model~\cite{michel2017global} instances of~\cite{Tourani_2018_ECCV} ($32$ instances with $600$-$4800$ variables and $13$ labels).\\
\hline
\end{tabular}
\caption{Datasets of increasing connectivity used for benchmarking.\label{tab:datasets}}
\end{figure*}

\subsection{Experimental Setup}\label{sec:exp-setup}
For a uniform evaluation over different problem types we formed three datasets grouping problems from different domains by their graph {\em connectivity}, which is the proportion of number of edges to the maximal possible number of edges $|\V||\V-1|/2$. These datasets are detailed in~\cref{tab:datasets}.

For objectiveness of comparison, we measure the computation cost in {\em messages}, the updates of the type $\min_{t} (a(t) + \theta_{uv}(s,t))$ that form the bulk of computation for all presented BCA methods.

The number of messages is scaled by the ratio $\overline{|\E|}/|\E|$, where $|\E|$ is the number of edges in an instance and $\overline{|\E|}$ is the average over the dataset. These normalizations allow us to show average performance on the whole datasets.

\subsection{TRWS vs. Subgraph Updates}\label{sec:TRWSvs}
\TRWS is selected as representing the most efficient node-adjacent update. In particular, it is much faster than \CMP and \MSD as shown \eg in~\cite{kolmogorov2015new,kappes-2015-ijcv}. It was originally derived as a method for optimizing the dual decomposition of~\eqref{equ:energy-min}  with monotonic chains~\cite{kolmogorov2006convergent}. We compared it to subgraph-based updates running on the same set of chains. Such direct comparison over datasets of different sparsity has not been conducted before. For a given graph we took a subset of maximum monotonic chains to cover all edges (exact details can be found in~\cref{sec:algs-suppl}). \TRWS is very efficient and takes $O(|\E|)$ messages to achieve optimality on all monotonic chains, including our covering subset. Two subgraph-based updates can be applied to optimize over chains in the covering subset sequentially: \TBCA, taking $O(|\E|)$ messages as well and hierarchical minorant (\HM,~\cref{alg:dmm}) taking $O(|\E|\log |\E|)$ messages. 
The comparison in~\cref{fig:TRWS-vs-subgraph} on our broad corpus of problems shows that while \TBCA is clearly inferior in performance to \TRWS, \HM is actually performing significantly better than \TRWS. It works on the same subproblems as \TBCA but the maximal minorant property justifies the extra computation time.

In this experiment we also evaluated an improved version of \TBCA, denoted \TBCAPP, which modifies \TBCA as follows: After each \trDP update on $uv$ the operation~\eqref{equ:handshake-1} pushes the remaining cost back to $v$ and thus achieves the maximal minorant conditions~\eqref{equ:max-minorant-condition}. This small change leads to a noticeable improvement, see~\cref{fig:TRWS-vs-subgraph}. However, we can conclude that a better redistribution of cost excess done by \HM is more important than the maximality alone.

\subsection{Static vs. Dynamic}
In the \TBCA and \HM methods, the choice of subproblems is not limited to monotonic chains. In~\cite{tarlow2011dynamic} it was proposed to select spanning trees dynamically favouring node-edge pairs with the most disagreement as measured by the local primal-dual gap. We verified whether this strategy is beneficial with \HM. \cref{fig:local-pd-gap} shows the comparison of dynamic spanning trees versus static spanning trees (a fixed collection selected greedily to cover all edges, see~\cref{sec:algs-suppl}). We reconfirm observations~\cite{tarlow2011dynamic} on our corpus of problems that dynamic strategy is beneficial with \TBCA updates. It does not however have a significant impact on the performance of \HM updates. We therefore propose to use a static collection of subgraphs, optimized for a given graph.
\subsection{Graph Adaptive Chain Selection}
\label{ssec:graph-adaptive}
\citet{Tourani_2018_ECCV} have shown that for densely connected graphs, edge-based updates are much faster than other methods. Since the \MPLPPP update is equivalent to \HM update on chains of length 1, this suggests that shorter chains are more beneficial in dense graphs. Intuitively, when there is a direct edge between nodes, the longer connections through other nodes become increasingly less important. On the contrary, in grid graphs \MPLPPP is found inferior to \TRWS~\cite{Tourani_2018_ECCV} and the natural choice of row and column chains seems to be the best selection of sub-problems. 

Based on these observations there is a need for the sub-graphs to be chosen adaptively to the graph topology. We use chain subproblems for their simplicity and better parallelization utility and propose the following informed heuristic:
\begin{itemize}[leftmargin=*]
\item Select subproblems sequentially as shortest paths from the yet uncovered part of the graph;
\item Chose {\em strictly} shortest paths, such that no other path of the same length connects the same nodes;
\item Find the most distant pair of nodes connected by a strict shortest path. 
\end{itemize}
An example of strict shortest path is given in~\cref{fig:strictlyShortestPath}. The algorithm implementing this heuristic, detailed in~\cref{sec:algs-suppl}, \cref{alg:compute-sp}, randomly picks a starting vertex, finds all strictly shortest paths from it (a variant of Dijkstra search), removes the longest traced shortest path from the graph and reiterates.

This heuristic has the following properties: (i) in a complete graph, it selects edge subproblems; (ii) in a grid graph, irrespective of the input data ordering, it is likely to select large pieces of rows and columns (with some distortions due to greediness); (iii) in graphs with bottleneck long connections, these connections are very likely to be covered with long chains. 

\subsection{The SPAM Algorithm}
The synthesis of block selection via  graph adaptive chain selection, as described in~\ref{ssec:graph-adaptive}, and the hierarchical minorant updates we call the {\bf S}hortest {\bf P}ath {\bf A}daptive {\bf M}inorant (SPAM) algorithm. The adaptive chain selection has a linear complexity w.r.t. the size of the graph and takes only a fraction of a single iteration time of the main algorithm.

\subsection{Final Experimental Evaluation} \label{sec:experimental-validation}
We tested the proposed \SPAM algorithm against existing methods. 
\cref{fig:dualComplete} shows the summarized evaluation results. One can see that across all graph types \SPAM consistently does well. It automatically adapts to the density of the graph, reducing to \MPLPPP for complete graphs, where \TRWS struggles. In grid graphs, where \TRWS uses the natural ordering, \SPAM automatically finds sub-problems similar to rows and columns and achieves a significant improvement while \MPLPPP becomes inefficient. 
Detailed results per dataset and speed-up factors with confidence intervals are included in~\cref{suppl:experiments}.

\section{Conclusion}
We have reviewed, systematized and experimentally compared different variants of block-coordinate-ascent methods proposed to date. We have shown the updates for subgraph-based methods take the form of modular minorants and that maximal minorants outperform non-maximal ones. We experimentally compared existing methods as well as new combinations of basic components of BCA algorithms and synthesized a novel algorithm that is a synthesis of the best aspects of all methods. It additionally adopts block-size to the graph structure and delivers uniformly best performance across the tested datasets.

\textbf{Acknowledgements} This work was supported by the German Reserach Foundation (“Exact Relaxation-Based Inference in GraphicalModels”, DFG SA 2640/1-1) and the European Research Council (ERC European Unions Horizon 2020 research and innovation program, grant 647769). The computations were performed on an HPC Cluster at the Center for Information Services and High Performance Computing (ZIH) at TU Dresden. Alexander Shekhovtsov was supported by the project ``International Mobility of Researchers MSCA-IF II at CTU in Prague'' $(CZ.02.2.69/0.0/0.0/18\_070/0010457)$.

{
\bibliographystyle{unsrtnat}
\bibliography{main-final}
}

\clearpage
\onecolumn
\appendix
\numberwithin{figure}{section}
\addtocontents{toc}{\protect\setcounter{tocdepth}{2}}
\pagestyle{empty}

\setcounter{figure}{0}
\setcounter{table}{0}
\counterwithin{figure}{section}
\counterwithin{table}{section}
\counterwithin{theorem}{section}
\counterwithin{proposition}{section}
\counterwithin{lemma}{section}

\begin{center}
\Large Designing an efficient dual solver for discrete energy minimization \\ Appendix
\end{center}

Contents:
\begin{itemize}
\item[A] Proofs of Theorems 1,2.
\item[B] Algorithms details and description of monotonic chains used in experiment~\cref{fig:TRWS-vs-subgraph}.
\item[C] Detailed experimental results.
\end{itemize}

\section{Proofs}\label{sec:proofs}

\Tminorant*

\begin{proof}

\paragraph{The "if" part} Let $\phi$ be a dual optimal reparameterization for $E_{\SG'}$ on the graph $\SG'	= (\V',\E')$. 
We need to show that $g(y)$ is a minorant. \ie
\begin{itemize}
\item $g(y) \leq E_{\SG'}(y)$ for all $y$. (\texttt{lower-bound property}) 
\item   $g(y^*) = E_{\SG'}(y^*)$ is the cost of a minimizing labeling $y^*$. (\texttt{same-minima property})
\end{itemize}
We have by the definition of reparametrization
\begin{eqnarray}	
\sum_{u \in  \SV'}\theta^{\phi}_u(y_u)+\sum_{uv \in  \SE'}\theta^{\phi}_{uv}(y_u,y_v)=\sum_{u \in  \SV'}\theta_u(y_u)+\sum_{uv \in  \SE'}\theta_{uv}(y_u,y_v)=E_{\SG'}(y) \label{equ:pr1repa1}.
\end{eqnarray}

We assume \Wlog $\theta^{\phi}_{uv}(y_u,y_v)\geq0$. Substituting this in~\eqref{equ:pr1repa1} we have 
\begin{equation}
g(y)=\sum_{u \in  \SV'}\theta^{\phi}_u(y_u) \leq \sum_{u \in  \SV'}\theta^{\phi}_u(y_u)+\sum_{uv \in  \SE'}\theta^{\phi}_{uv}(y_u,y_v) = E_{\SG'}(y), \forall y \in \Y^{n} \implies g(y) \leq  E_{\SG'}(y), \forall y \in \Y^{n} \label{equ:pr1minorant1},
\end{equation}
where the left hand side matches the definition of $g(y)$ in the theorem. With this we have proved the lower bound property.

Now we prove the same minima property. 
Comparing the dual function~\eqref{equ:Lagrange-dual} with $g(y)$ and $g(y)$ with ~\eqref{equ:energy-min} we have the following inequalities:
\begin{eqnarray}
D(\phi)= \sum_{u \in \SV'}\min_{y_u}\theta^{\phi}_u(y_u)+\sum_{uv \in  \SE'}\min_{y_u,y_v}\theta^{\phi}_{uv}(y_u,y_v) \leq \sum_{u \in \SV'}\theta^{\phi}_u(y_u)+\sum_{uv \in  \SE'}\min_{y_u,y_v}\theta^{\phi}_{uv}(y_u,y_v)=g(y), \forall y \in \Y^{n} \label{equ:D_lower_g};\\
g(y)=\sum_{u \in \SV'}\theta^{\phi}_u(y_u)=\sum_{u \in \SV'}\theta^{\phi}_u(y_u)+\sum_{uv \in  \SE'}\theta^{\phi}_{uv}(y_u,y_v) \leq E_{\SG'}(y), \forall y \in \Y^{n} \label{equ:g_lower_E}.
\end{eqnarray}
Since $\SG'$ is a tree-subgraph, strong duality holds and we have for all pairs of an optimal labeling $y^*$ and an optimal dual $\phi$ that $D(\phi)=E_{\SG'}(y^{*})$ and there holds complementarity slackness conditions. It follows that $\min_{y_u}\theta^{\phi}_u(y_u)$ is attained at $y^*_u$ and $\min_{y_u, y_v}\theta^{\phi}_{u,v}(y_u, y_v)$ is attained at $(y^*_u, y^*_v)$ (there is an optimal solution composed of minimal nodes and edges). It follows that the next inequalities are satisfied:
\begin{align}
\theta^{\phi}_u(y_u) & \geq \theta^{\phi}_u(y^*_u), \label{t1-ieq-1} \forall y_u\in\Y;\\
\theta^{\phi}_u(y_u, y_v) & \geq \theta^{\phi}_{u,v}(y^*_u, y^*_v) , \ \forall y_u,\, y_v\in\Y.
\end{align}
Using~\eqref{t1-ieq-1} in $g(y)$ we obtain 
\begin{align}
g(y) \geq \sum_{u}\theta^{\phi}_u(y^*_u) = E_{\SG'}^*.
\end{align}
Thus as $E_{\SG'}(y^{*}) \leq g(y^{*})\leq E_{\SG'}(y^{*})$,  $g(y^{*})= E_{\SG'}(y^{*})$, proving the \texttt{equal-minima} property.

\paragraph{The "only if" part} 
We have to show that if $g(y)$ is a minorant of $E_{\SG'}$, then $\phi$ is an optimal reparameterization, \ie . $D(\phi)=E_{\SG'}(y^*)=g(y^*)$, where $y^*$ is the optimal labelling for $E_{\SG'}$.

Due to the minorant \texttt{equal-minima property}, we have
\begin{equation}
g(y^*)=\sum_{u \in \V'}\theta^{\phi}_u(y^*_u) =\sum_{u \in \V'}\theta^{\phi}_u(y^*_u) +\sum_{uv \in \SE'} \theta^{\phi}_{uv}(y^*_u,y^*_v) = E_{\SG'}(y \mid \theta^\phi) \implies \sum_{uv \in \SE'} \theta^{\phi}_{uv}(y^*_u,y^*_v)=0;
\end{equation}
As we assume $\theta^{\phi}_{uv}(s,t)\geq 0$, for all $s,t \in \SY^2$ and $uv \in \SE'$, this would imply all terms $\theta^\phi_{uv}(y^*_u,y^*_v)$ are identically zero, \ie
\begin{equation} \label{equ:prws-all-0}
\theta^\phi_{uv}(y^*_u,y^*_v)=0, \quad \forall uv \in \SE'.
\end{equation}
Our initial objective was to show $D(\phi)=g(y^*)=E(y^* \mid \theta^\phi)$. As we assume $\sum_{uv \in \SE'}\min_{s,t}\theta^\phi_{uv}(s,t)=0$, we just have to show
\begin{equation}
D(\phi)=\sum_{u \in \SV'}\min_s\theta^{\phi}_u(s)=\sum_{u \in \SV'} \theta^{\phi}_u(y^*_u)=g(y^*).
\end{equation}
Following a proof by contradiction argument, we claim
\begin{equation}
D(\phi)=\sum_{u \in \SV'}\min_s\theta^{\phi}_u(s)=\sum_{u \in \SV'}\theta^{\phi}_u(y_u^*)=g(y^*)=E_{\SG'}(y^* \mid \theta^\phi).
\end{equation}
Assume the above statement is false and let $D^*$ be the optimal dual. 

Further, \Wlog let's assume the $\min_s \theta^\phi_u(s)=y^*_u$ for all $u \in \SV' \setminus k$ and $\min_s \theta^\phi_k(s)=y_k^{+}$. As strong duality holds, we have 
\begin{eqnarray}
D(\phi)=\sum_{u \in \SV'}\min_s\theta^{\phi}_u(s)=\sum_{u \in \SV' \setminus k}\theta^{\phi}_u(y_u^*) + \theta^\phi_k(y^{+}_k)\label{equ:D-phi};\\
D^*=\sum_{u \in \SV'}\theta^{\phi}_u(y^*_u)=E_{\SG'}(y^* \mid \theta^\phi)\label{equ:D-opt-E}.
\end{eqnarray}
Thus by assumption $D(\phi)\geq D^*$, 
\begin{equation}
D(\phi)\geq D^* \implies \sum_{u \in \SV' \setminus k}\theta^{\phi}_u(y_u^*) + \theta^\phi_k(y^{+}_k) \geq \sum_{u \in \SV'}\theta^{\phi}_u(y^*_u) \implies \theta^\phi_k(y^{+}_k) \geq \theta^\phi_k(y^*_k).
\end{equation}
But $\theta^\phi_k(y^{+}_k)=\min_{y_k}\theta^{\phi}_k(y_k)\leq \theta^\phi_k(y^*_k)$, this is therefore a contradiction and $D(\phi)=D^*=g(y^\phi)=E(y^\phi \mid \theta^\phi)$.
\end{proof}

\Tmaxminorant*

\begin{proof}

\paragraph{"Only if part"}. 

For an optimal reparametrization $\phi$, its corresponding tight minorant by~\cref{prop:minorant-property} is $g(y)=\sum_{u}\theta^{\phi}_u(y_u)$. We need to prove the statement that minorant $g$ is maximal only if the conditions in the theorem are fulfilled.

Recall that we are working with the constrained dual so that $\theta^\phi \geq 0$ component-wise.
Assume for contradiction that one of the two zero minimum conditions is violated. Let it be the one with minimum over $s'$. Then $\exists uv\in\E'$ $\exists t$ such that $\lambda(t) := \min_{s'}\theta_{uv}(s',t) > 0$. We can then add $\lambda(t)$ to $\phi_{vu}(t)$. This will not destroy optimality of $\phi$ but will strictly increase $\theta^{\phi}_{v}(t)$, therefore leading to a strictly greater minorant, which contradicts maximality of $g$.

\paragraph{"If part"} We need to show that if the conditions of the theorem are fulfilled then $g$ is maximal.

Assume for contradiction that $g$ is not maximal, \ie there is a modular function $h(y)$ such that it is also a minorant for $E_{\G'}$ and it is strictly greater than $g$: $h(y) \geq g(y)$ for all $y$ and $h(y') > g(y')$ for some $y'$. 

The inequality $h(y) \geq g(y)$ for modular functions without constant terms is equivalent to component-wise inequalities:
\begin{align}
h_u(y_u) \geq g_u(y_u), \ \forall u, \forall y_u.
\end{align}
From the inequality $h(y') > g(y')$ we conclude that there exists $u$ and $y'_u$ such that $h_u(y'_u) > g_u(y'_u)$. By the conditions of the theorem, and assuming a tree graph, a labeling $y'$ can be constructed such that it takes label $y_u'$ in $u$ and all costs $\theta_{u,v}^\phi(y'_u, y'_v)$ are zero. The construction starts from $y'_u$, finds labels in the neighbouring nodes such that edge costs with them is zero and proceed recurrently with the neighbours and their unassigned neighbouring nodes. For the labeling $y'$ constructed in this way we have that
\begin{align}
g(y') = \sum_{u}\theta^{\phi}_{u}(y'_u) = \sum_{u}\theta^{\phi}_{u}(y'_u) + \sum_{uv}\theta^{\phi}_{u,v}(y'_u, y'_v) = E_{\G'}(y').
\end{align}
At the same time, $h(y') > g(y')$ and therefore $h(y') > \E_{\SG'}(y')$, which contradicts that $h$ is a minorant of $E_{\SG'}$.
\end{proof}

\section{Algorithms Details}\label{sec:algs-suppl}

\subsection{Maximal Monotonic Chains}

In this section we describe how we selected a collection of monotonic chains ($\texttt{MMC}$), on which \TRWS can run in its full efficiency and at the same time subgraph-based updates of \TBCA and \HM can be computed.

A chain is a subgraph of graph $\G=(\V,\E)$ that is completely defined by enumerating the sequence of nodes it contains, $\ie$ a chain $\SC$ is denoted as $\SC=(n_1,\hdots, n_M), n_i \in \V$, with $(n_i,n_{i+1}) \in \E$  for $i=1:M-1$ denoting the edges it contains. Therefore, for every pair of consecutive nodes $(n_i,n_{i+1})$ there must also exist a corresponding edge in $\E$ for a chain to be a subgraph of $\G$.

Let there be a partial order defined on the nodes $\V$ such for each edge $uv\in\E$ the nodes are comparable: either $u>v$ or $v<u$. This can be always completed to a total order as was used for simplicity in~\cite{kolmogorov2006convergent}. A chain $\SC$ is said to be {\em monotonic} if $n_i < n_{i+1}$ holds for its nodes. A chain $\SC$ is {\em maximal monotonic} if it is monotonic and not a a proper subgraph of some other monotonic chain.



For a given ordering, we select a collection of edge disjoint monotonic chains covering the graph by greedily finding and removing from the edge set maximal monotonic chains. Finding and removing one chain is specified by~\cref{alg:compute-mmc}. The algorithm works on the graph adjacency list representation. Let $Ad$ be the adjacency list corresponding to the directed version of directed the graph $\G$: $Ad(i)$ contains all neighbours of node $i$ in $G$ that are greater than $i$, \ie $\forall j \in Ad(i), j >i$. The operation $Ad(i).remove(j)$ removes element $j$ from the list $Ad(i)$. The algorithm is executed until all $Ad$ lists are empty (all edges have been covered).

\begin{algorithm}[h!] 
\caption{Compute Maximal Monotonic Chain}
\begin{algorithmic}[1]
\Function{$(\SC,Ad)$=computeMMC}{$Ad$}
\Comment{$Ad$ is the adjacency list of $\G$ as defined above.}
\State $\SC=\emptyset$, $\texttt{tail}=\emptyset$, $done=false$ \Comment{$\SC$ is initially empty., $\texttt{tail}$ is the last node added to the chain.}
\State Find the smallest in the order $i$ such that $Ad(i)$ is not empty.
\State $\SC.add(i)$, $\texttt{tail}=i$. \Comment{Add node $i$ to $\SC$. Update \texttt{tail}.}
\While{$!done$}
\State Find $j$ in $Ad(\texttt{tail})$ such that $j>\texttt{tail}$.
\If{$j$ is found}
\State $\SC.add(j)$, $Ad(\texttt{tail}).remove(j)$, $\texttt{tail}=j$ \Comment{The node $j$ is added to $\SC$, removed from $Ad(\texttt{tail})$. \texttt{tail} is updated.}
\ElsIf{$j$ is not found}
\State $done=true$	\Comment{The loop exit condition is satisfied.}
\EndIf
\EndWhile
\EndFunction
\end{algorithmic}
\label{alg:compute-mmc}
\end{algorithm}

The result of the algorithm is a collection of chains that are monotonic \wrt to the ordering. TRWS running on the respective ordering of nodes as introduced in~\cref{sec:update-types} can be viewed also as optimizing the dual decomposition with monotonic chains~\cite{kolmogorov2006convergent}. It can be shown that the number $\max(N_{\rm in}(u), N_{\rm out}(u)$ used to calculate weights in TRWS is exactly the number of different chains containing node $u$ for any collection of monotonic chains found as above. Hence such a collection natively represent subproblems associated with TRWS.

\subsection{Message Passing in Spanning Trees}
The hierarchical minorant for chains involves passing messages from the ends of the chain to the central nodes, as shown in~\ref{alg:dmm}. For trees, the process is similar. Messages are passed from the leaf nodes to the central nodes. The centroid of a tree of size $n$ is the node whose removal results in subtrees of size $\leq\floor{\frac{n}{2}}$. The central nodes of a tree are defined as nodes connected by an edge whose removal gives trees that are similar in length. One of the central nodes is always the tree-centroid. The other inode s selected keeping in mind minimum deviation between the different sub-trees that arise from the removal of this node. As the hierarchical minorant is recursive, the recursion is repeated with a subtree.

\subsection{Generation of Spanning Trees in TBCA}
For the static strategy, we compute a sequence of minimum weight spanning trees with the weights being the number of times an edge has already been included in a spanning tree. This weighing scheme ensures that un-sampled edges are prioritized in building spanning trees. The sampling is stopped when all the edges are covered. In the experiments (below) we observed that with the block update strategy that we chose, dynamic updates were not advantageous any more and performed slower overall.

\begin{algorithm}[h!] 
\caption{Compute Strictly Shortest Path}
\begin{algorithmic}[1]
\Function{$(\mathcal{C})$=computeSSP}{$\SG=(\SV,\SE)$,src}
\Comment{\emph{src} is the source node from which to grow the shortest path.}
\State Create Vertex Set $\SQ$ from graph $\SG$
\For{Each Vertex $v$ in $\SG$}
\State $dist[v]:=\infty$	\Comment{Set distance of all vertices to $\infty$}
\State $prev[v]:=UNDEFINED$ \Comment{Initialize all previous nodes to default value.}
\EndFor
\State $dist[src]=0$	\Comment{Distance from the source node to the source node is 0}
\While{$\SQ$ is not empty}
\State $u:=$ vertex in $\SQ$ with min $dist[u]$ 	\Comment{$u$ is assigned vertex in $\SQ$ with minimum value in $dist[$ $]$} 
\For{each neighbor $v$ of $u$}	\Comment{Only $v$ that are still in $\SQ$}
\State $alt:=dist[u]+1 $	\Comment{$alt$ is $dist[u]+length(u,v)$, which equals $1$}
\If{$alt<dist[v]$}	\Comment{If $dist[v]$ is greater than alt update distance}
	\State $dist[v]:=alt$
	\State $prev[v]:=u$
\ElsIf{$alt==dist[v]$}	\Comment{Condition for strictness of shortest path is violated}
\State	\texttt{break}
\EndIf
\EndFor
\EndWhile
\State Construct chain $\SC$ from $prev[$ $]$
\EndFunction
\end{algorithmic}
\label{alg:compute-sp}
\end{algorithm}

\section{Detailed Experimental Results}\label{suppl:experiments}

We show in~\cref{fig:suppl-plate} results per individual application, with performance in both messages and time. Since in each application, there are still multiple instance, we apply the same normalization and averaging procedures as in the main paper.

\begin{figure*}[t]
\centering
\def\fscale{0.14}
\setlength{\tabcolsep}{2pt}
\renewcommand\arraystretch{1.1}
\begin{tabular}{|c c c|}
\hline
& \texttt{stereo} & \texttt{mrf-inpainting} \\
\parbox[t]{2mm}{\multirow{2}{*}{\rotatebox[origin=c]{90}{\bf Sparse}}}&
\includegraphics[scale=\fscale]{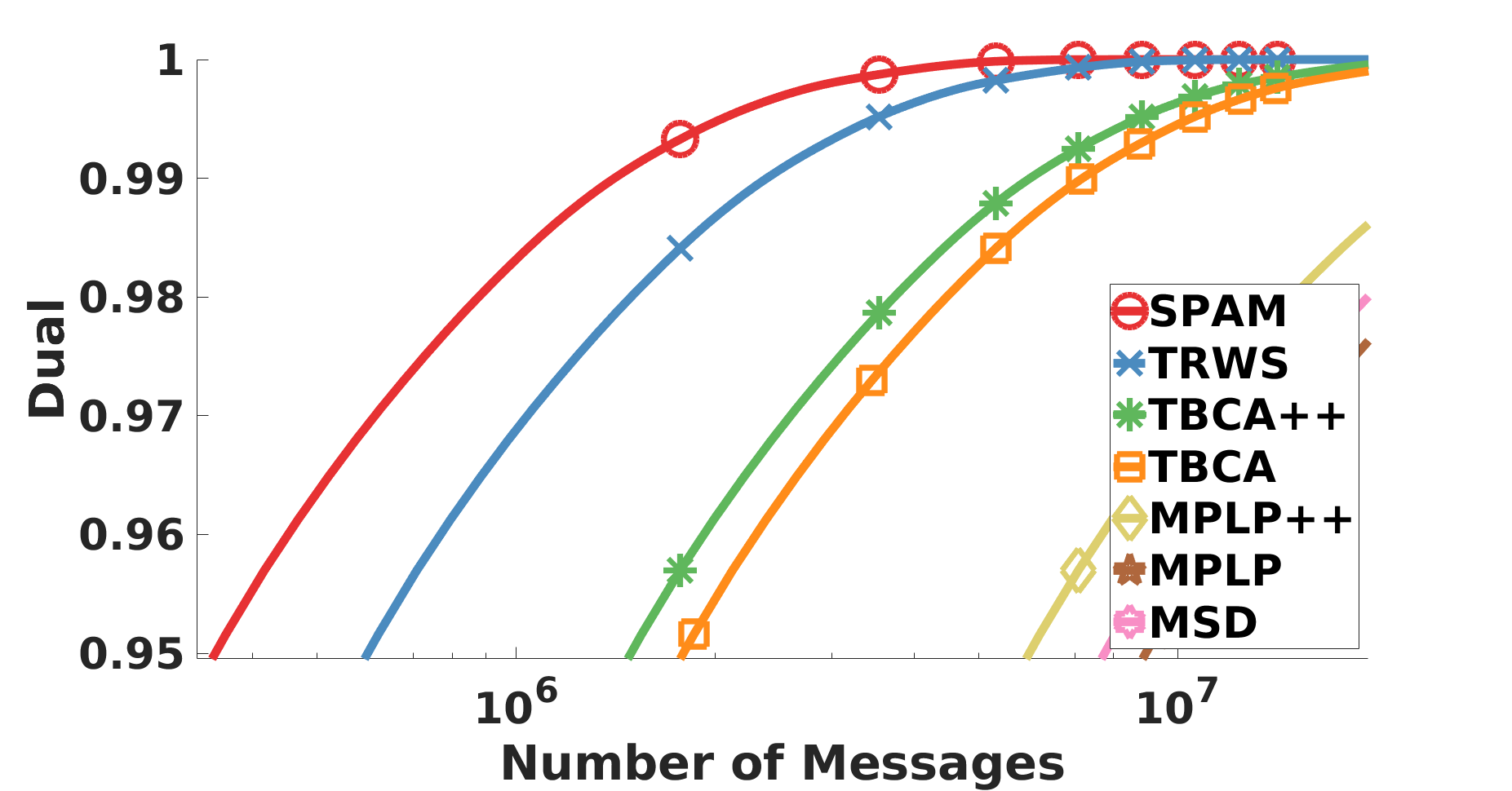}&%
\includegraphics[scale=\fscale]{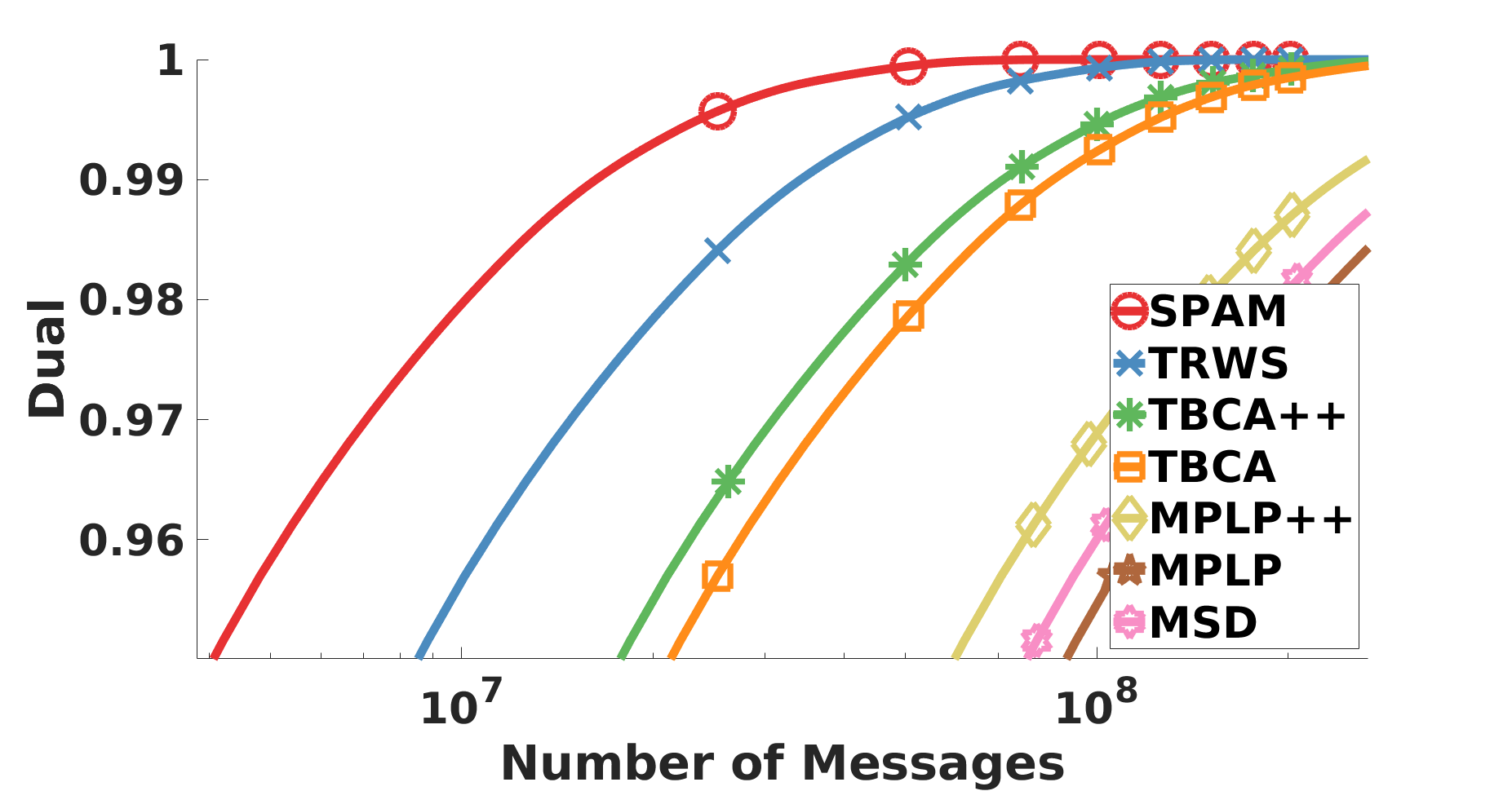}%
\\
& \includegraphics[scale=\fscale]{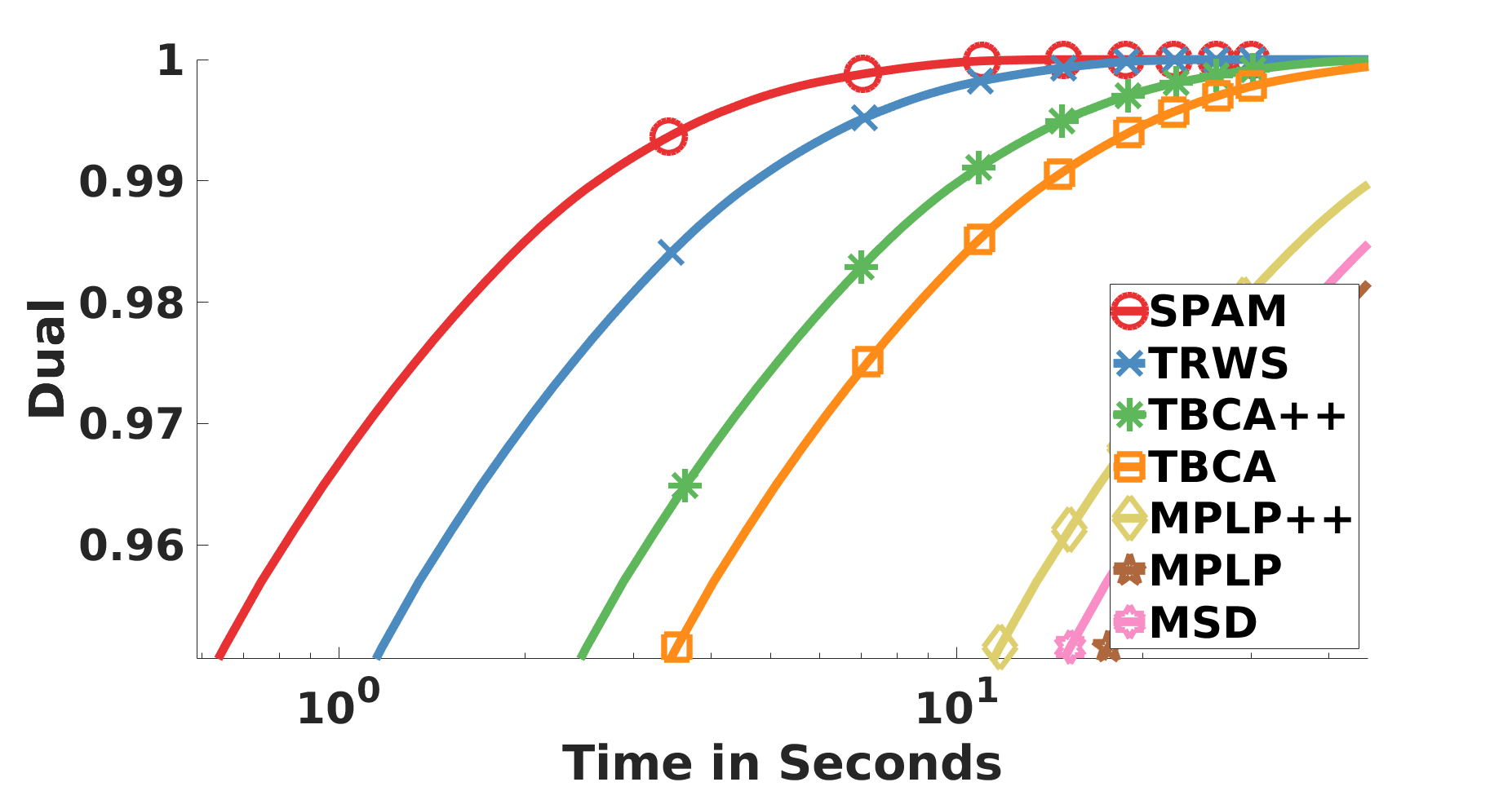}&
\includegraphics[scale=\fscale]{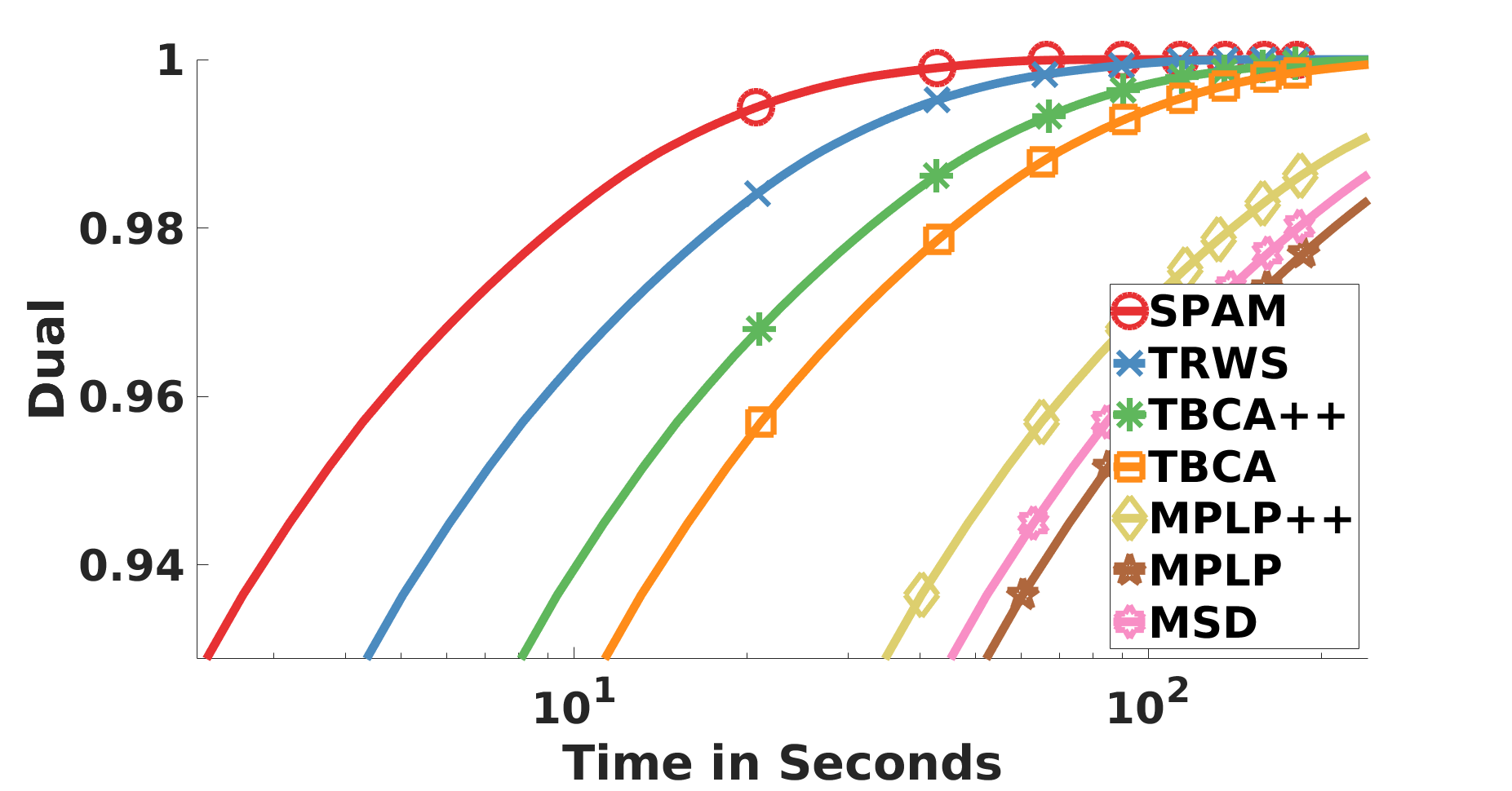}%
\\
\hline
& \texttt{denser-stereo} & \texttt{worms} \\
\parbox[t]{2mm}{\multirow{2}{*}{\rotatebox[origin=c]{90}{\bf Denser}}}&
\includegraphics[scale=\fscale]{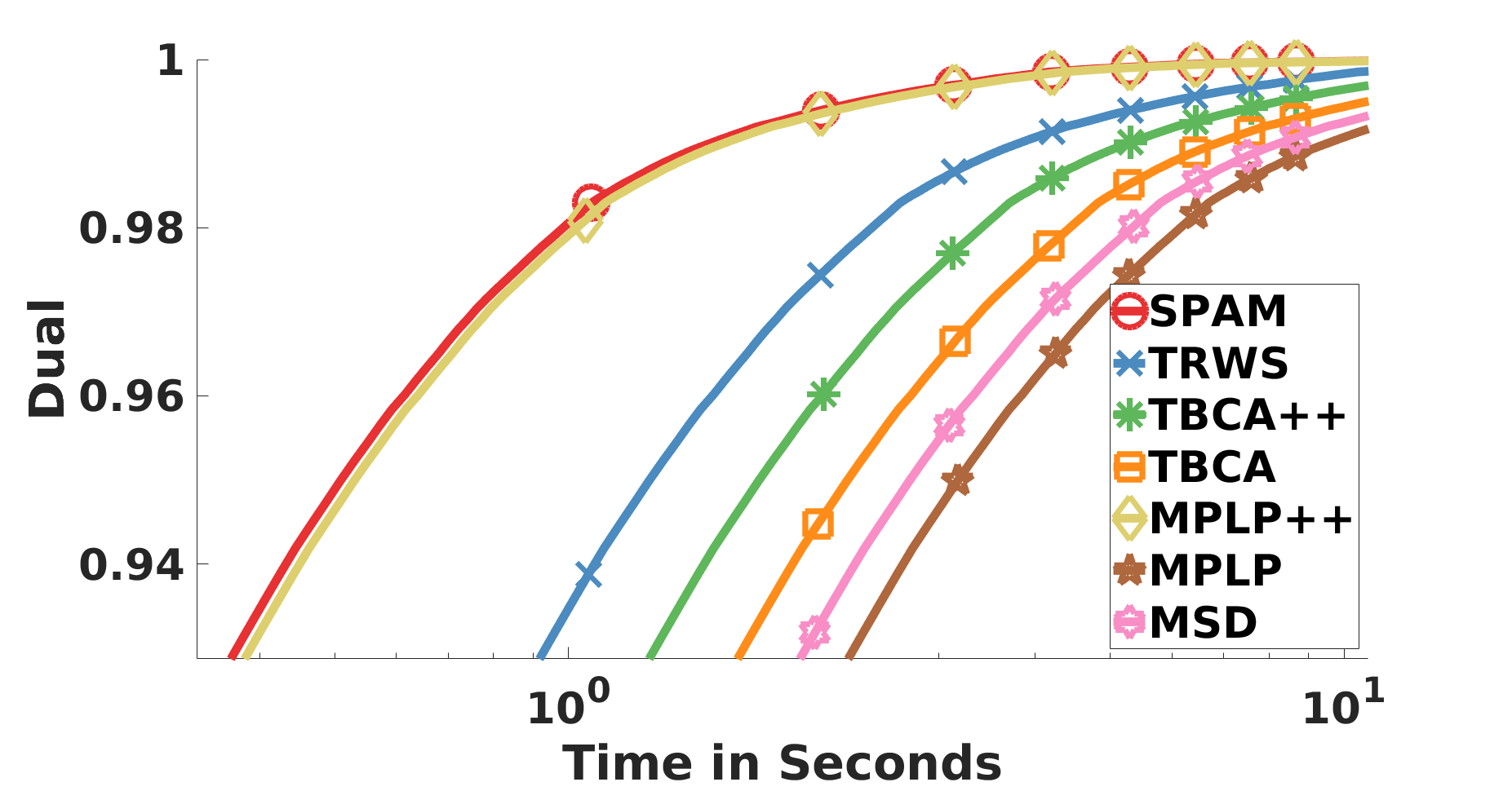}&%
\includegraphics[scale=\fscale]{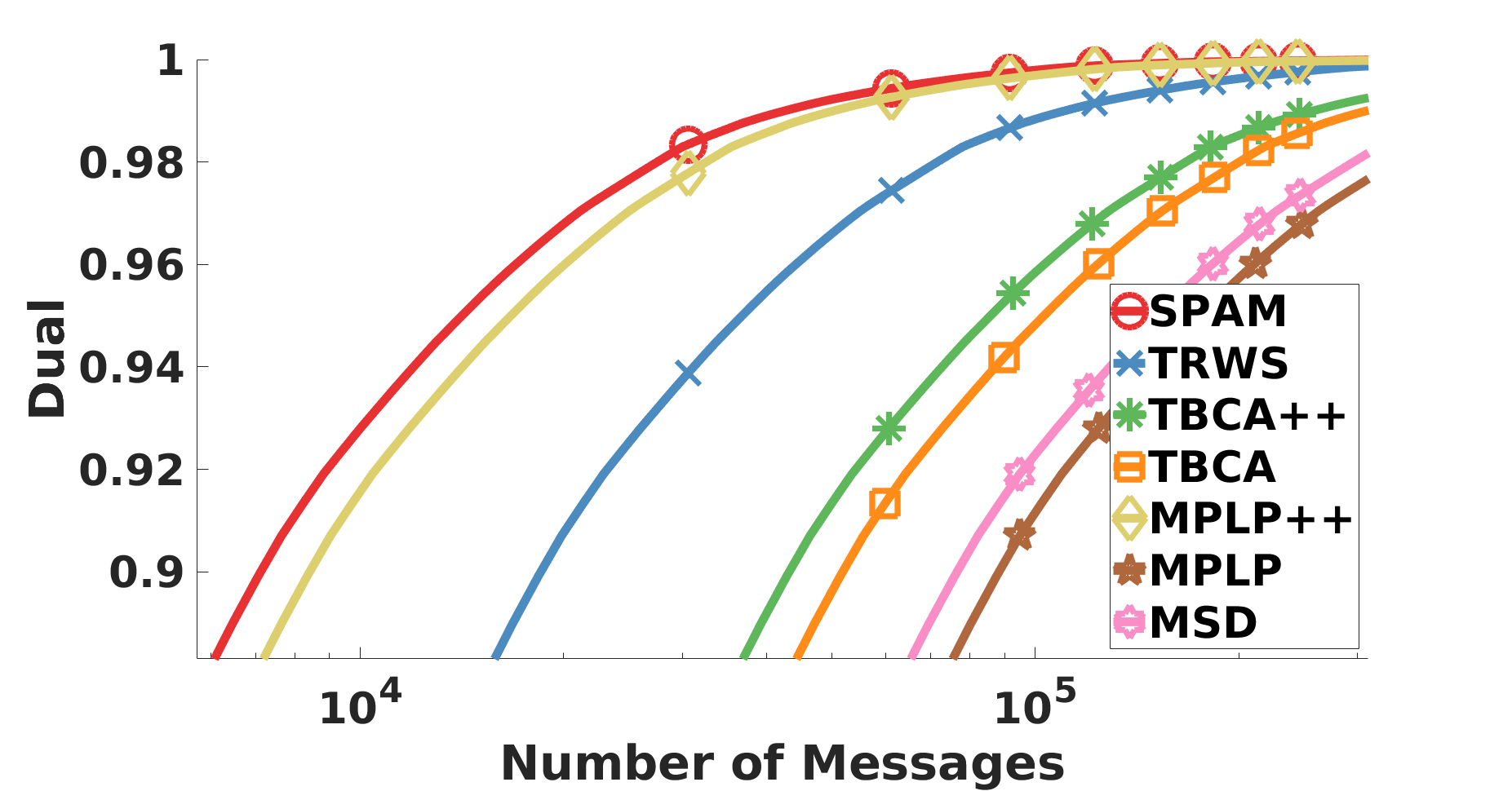}%
\\
& \includegraphics[scale=\fscale]{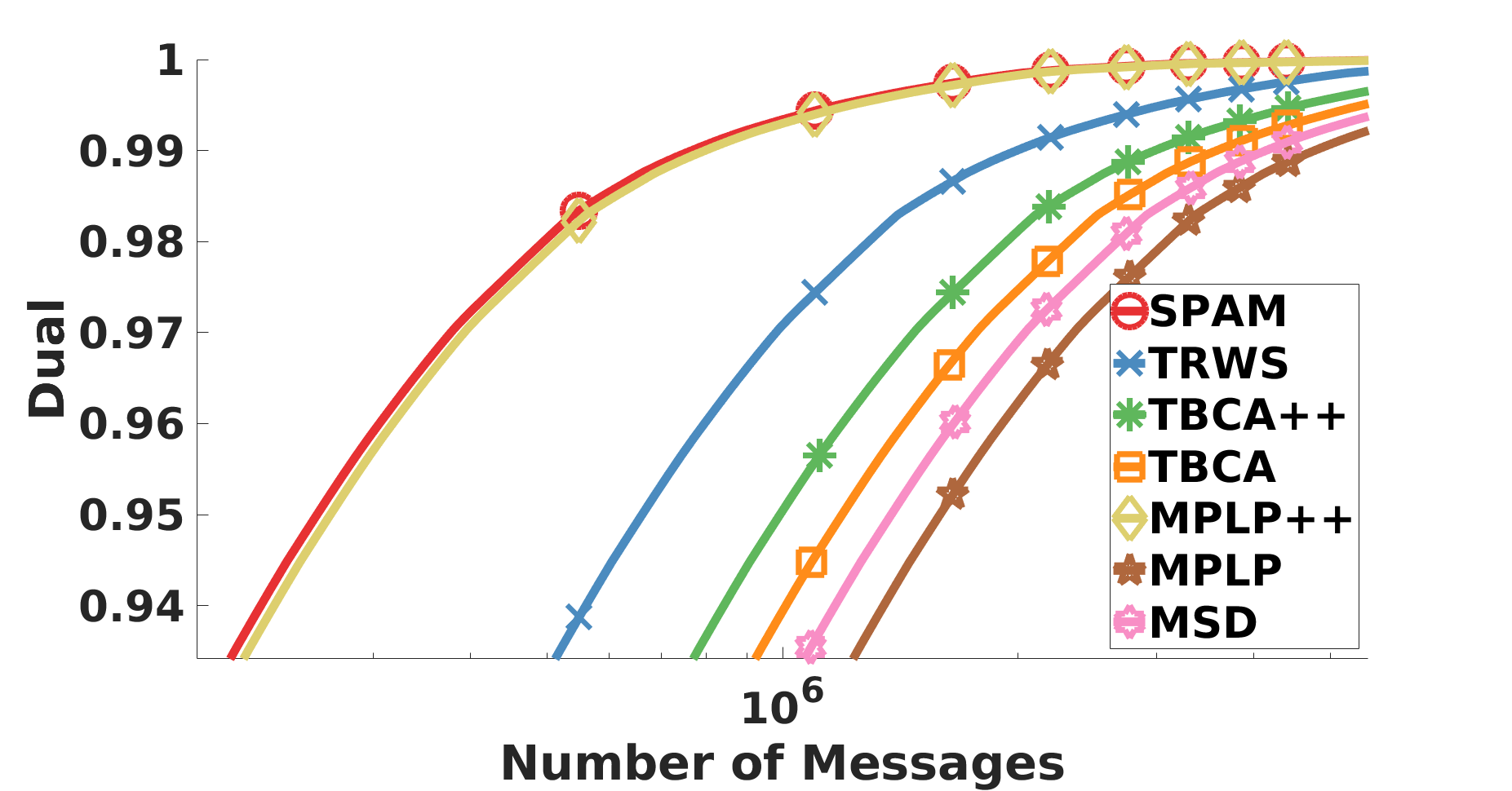}&%
\includegraphics[scale=\fscale]{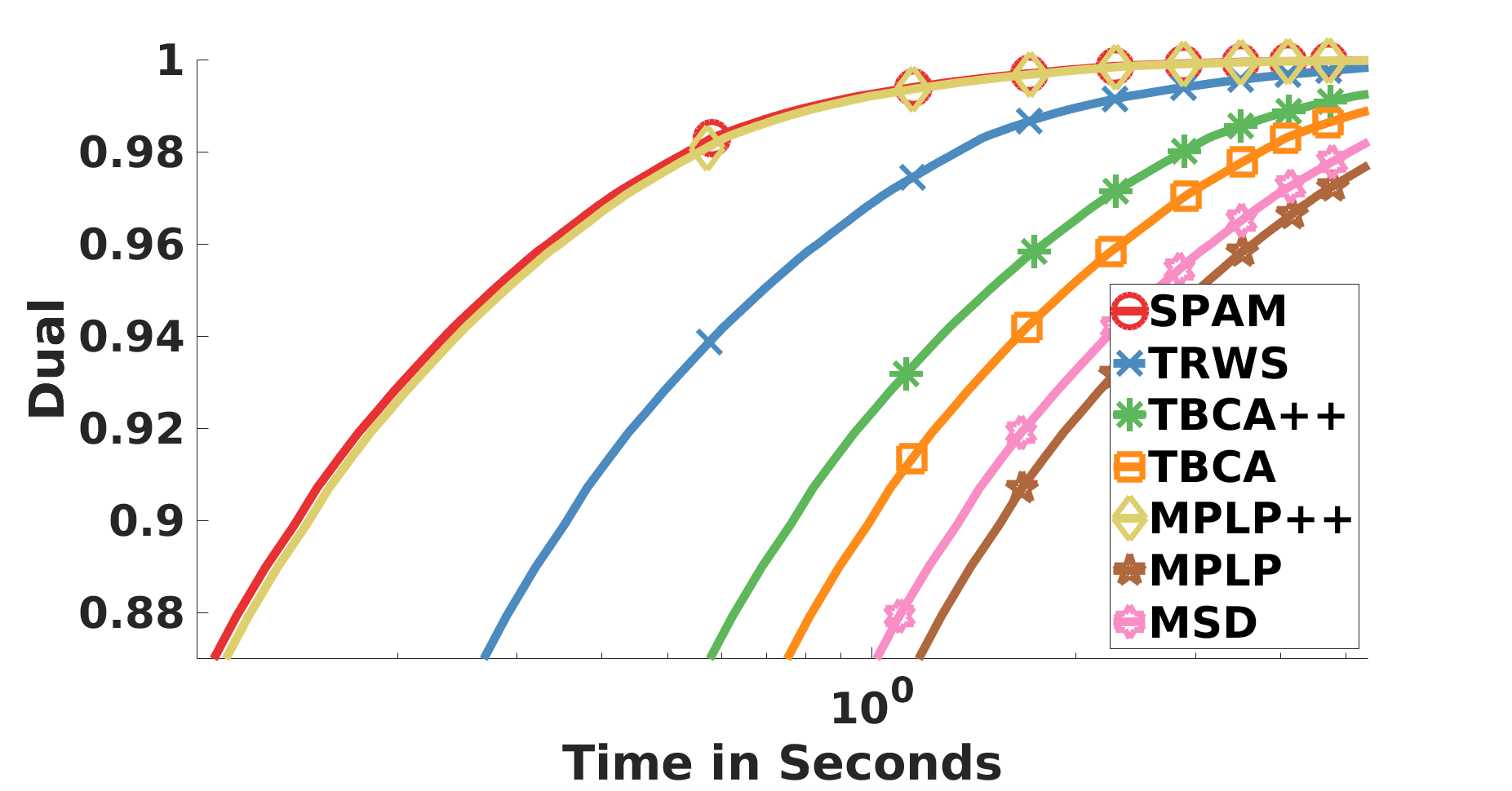}%
\\
\hline
& \texttt{protein} & \texttt{pose} \\
\parbox[t]{2mm}{\multirow{2}{*}{\rotatebox[origin=c]{90}{\bf Complete}}}&
\includegraphics[scale=\fscale]{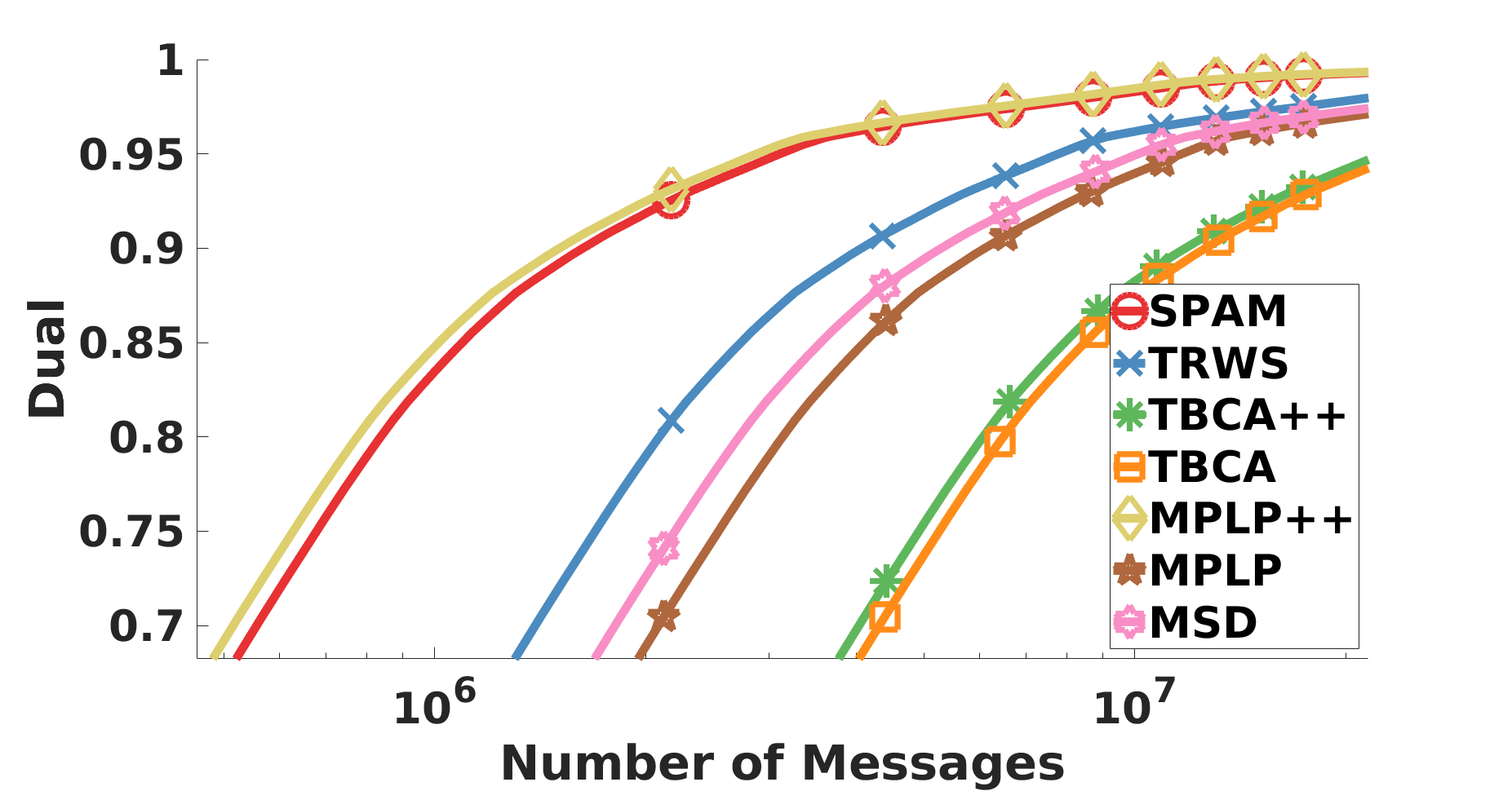}&%
\includegraphics[scale=\fscale]{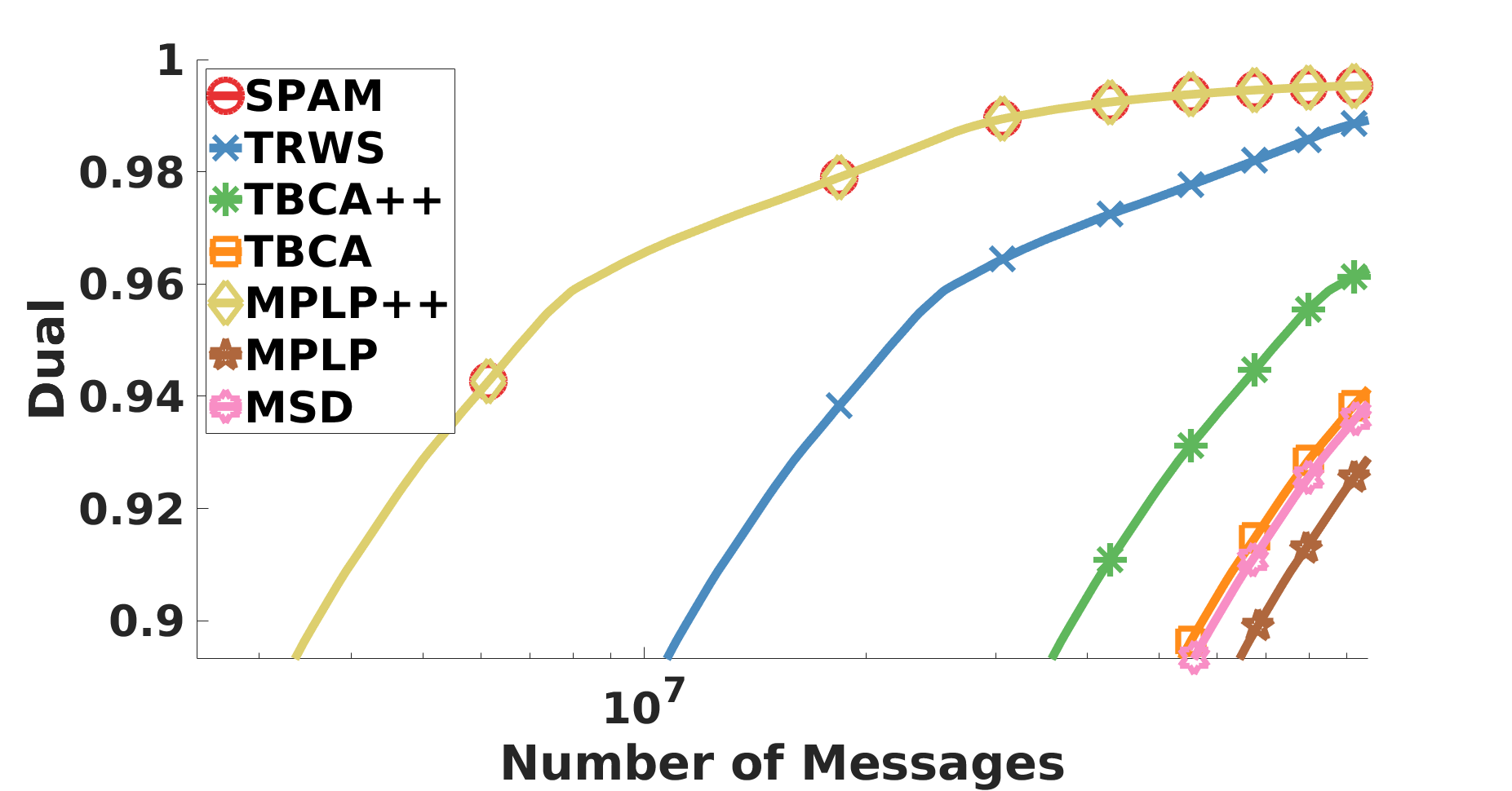}%
\\
& \includegraphics[scale=\fscale]{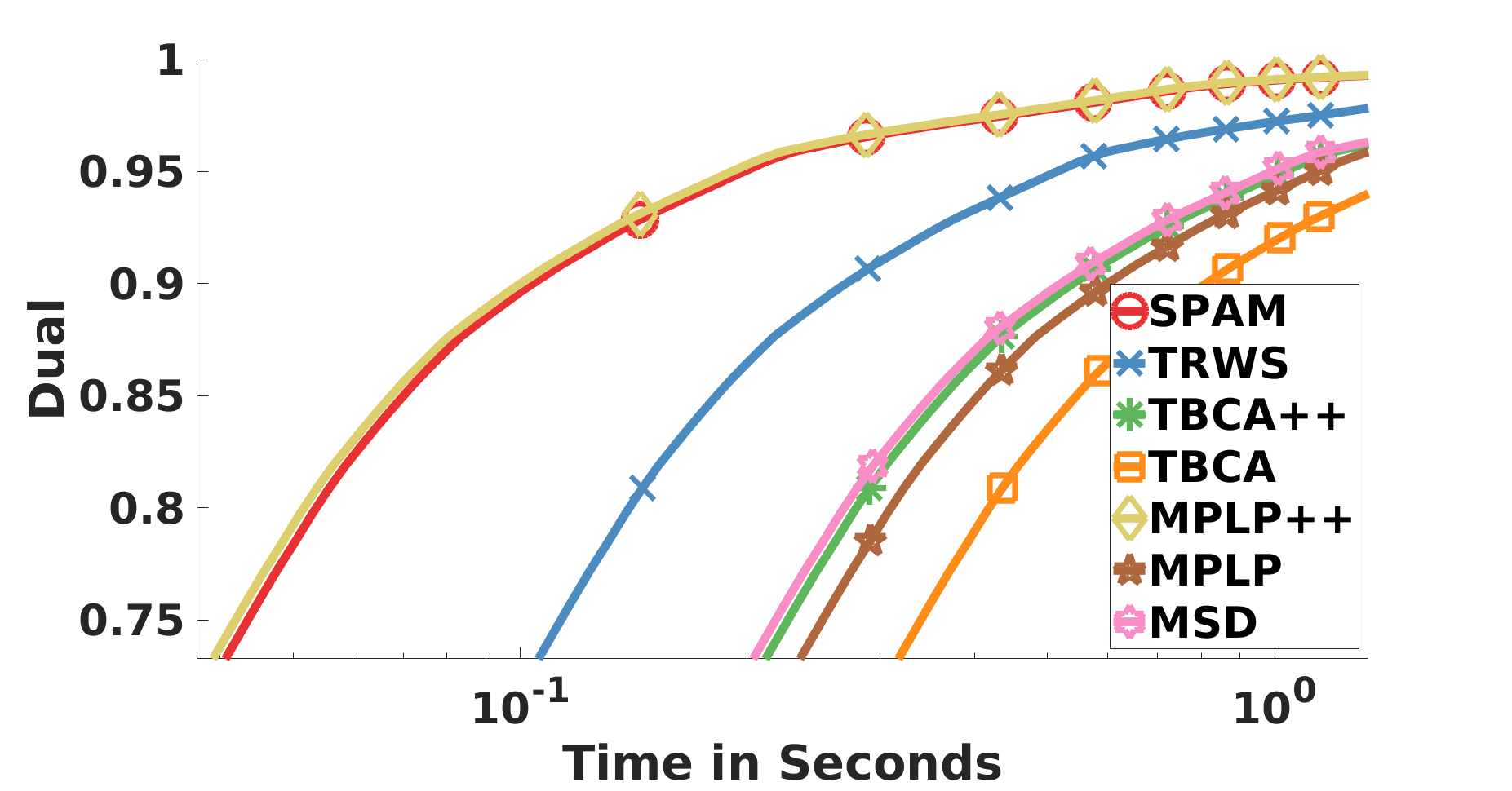}&
\includegraphics[scale=\fscale]{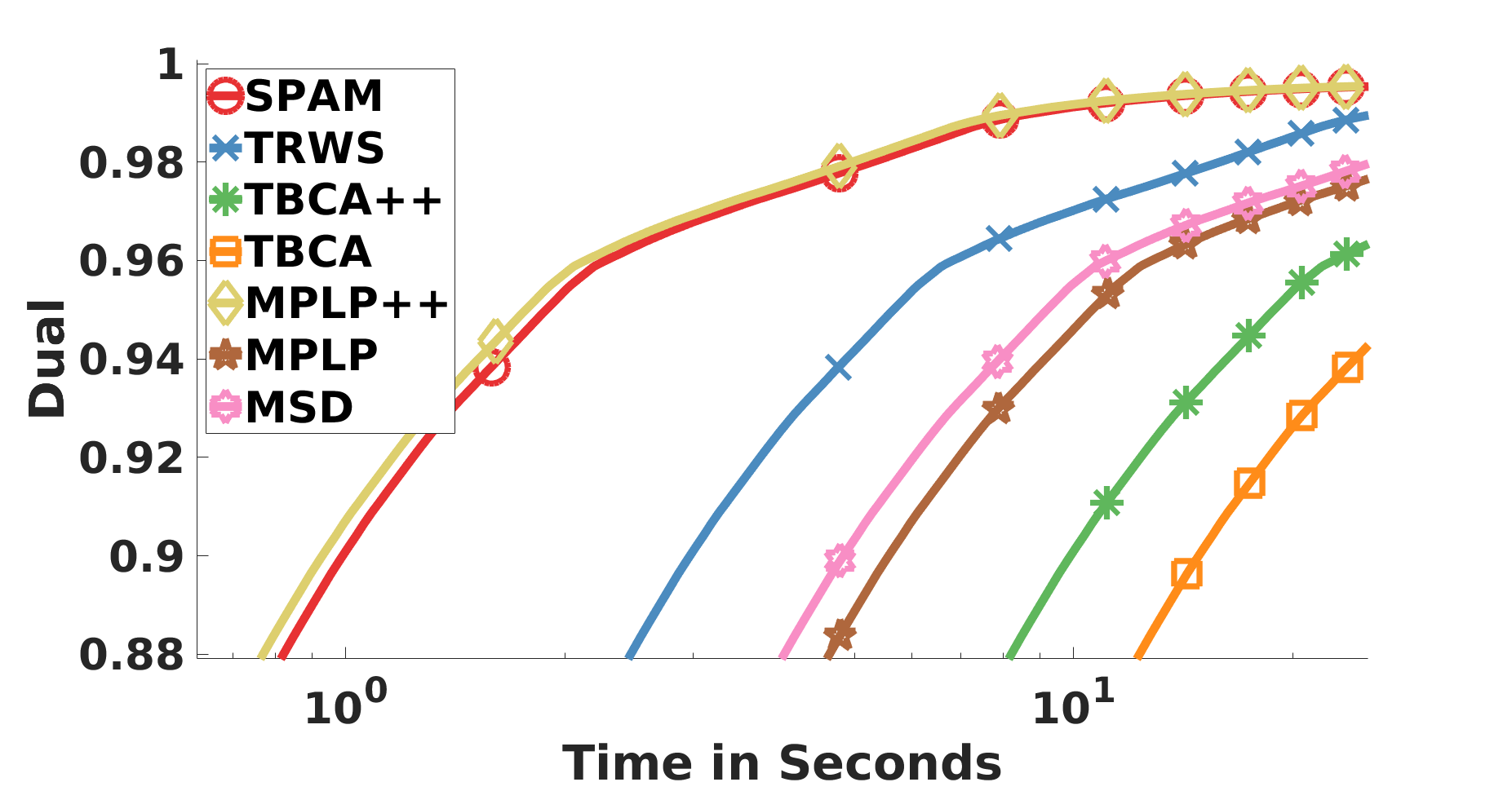}\\
\hline
\end{tabular}
\caption{The averaged plots for application-specific datasets: messages and time.\label{fig:suppl-plate}}
\end{figure*}

\begin{figure*}[t]
\centering
\def\fscale{0.14}
\setlength{\tabcolsep}{2pt}
\begin{tabular}{|c c|}
\hline
\text{\SPAM/\TRWS on sparse graphs} & \text{\SPAM/\TRWS on denser graphs}\\ 
\includegraphics[scale=\fscale]{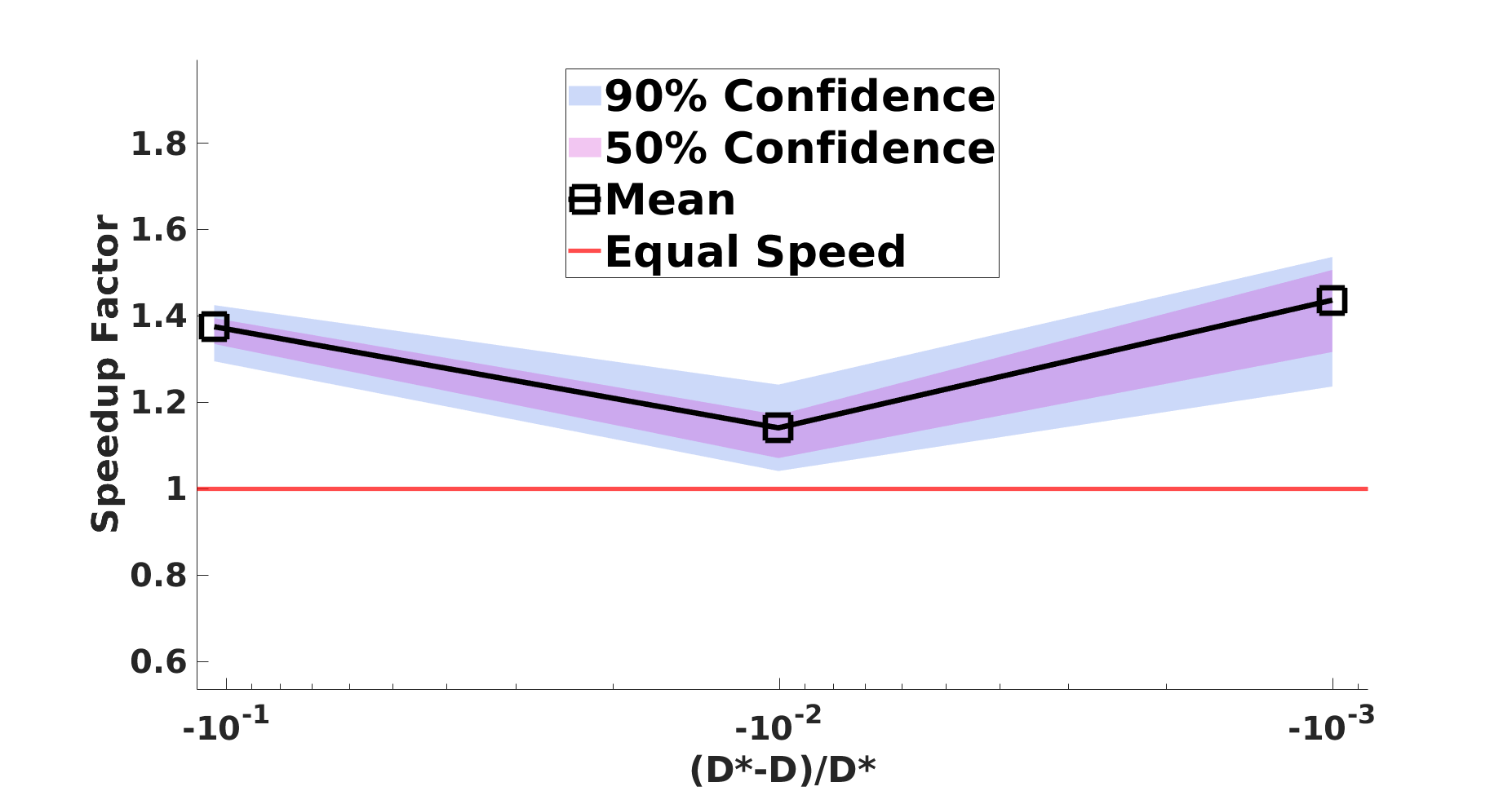}&%
\includegraphics[scale=\fscale]{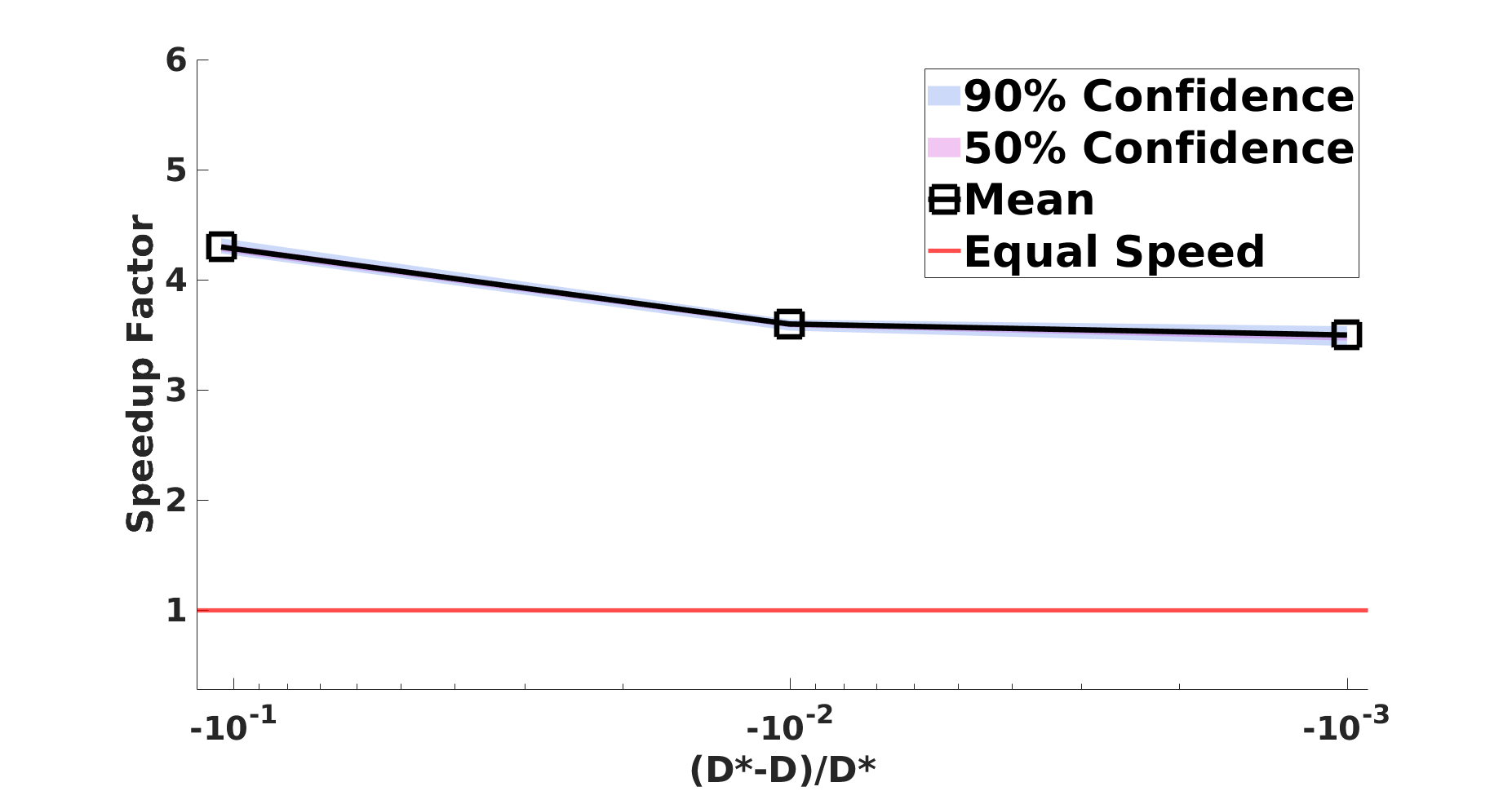}%
\\
\hline
\text{\SPAM/\TRWS on complete graphs} & \text{\SPAM/\MPLPPP on sparse graphs}\\ 
\includegraphics[scale=\fscale]{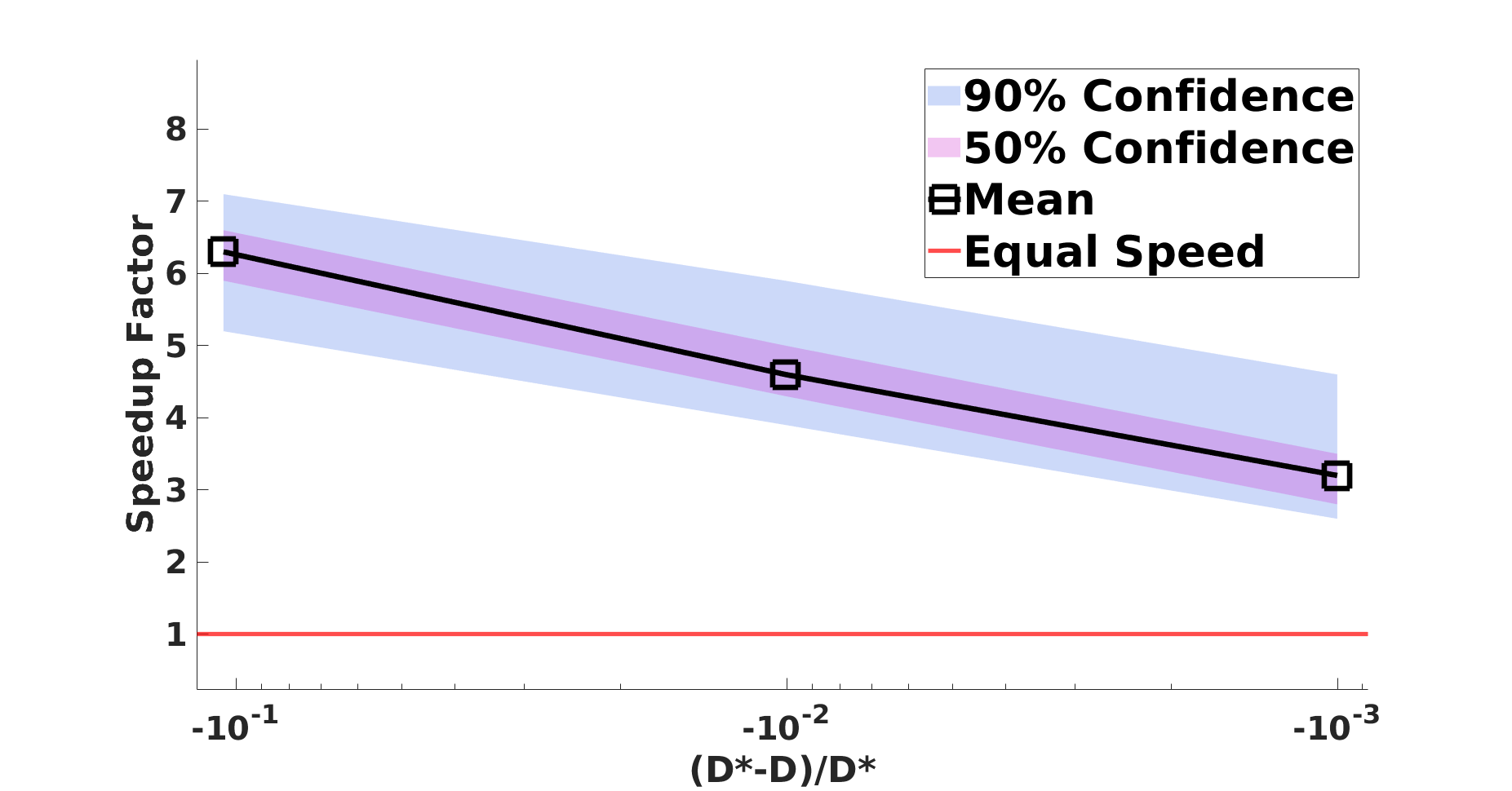}&%
\includegraphics[scale=\fscale]{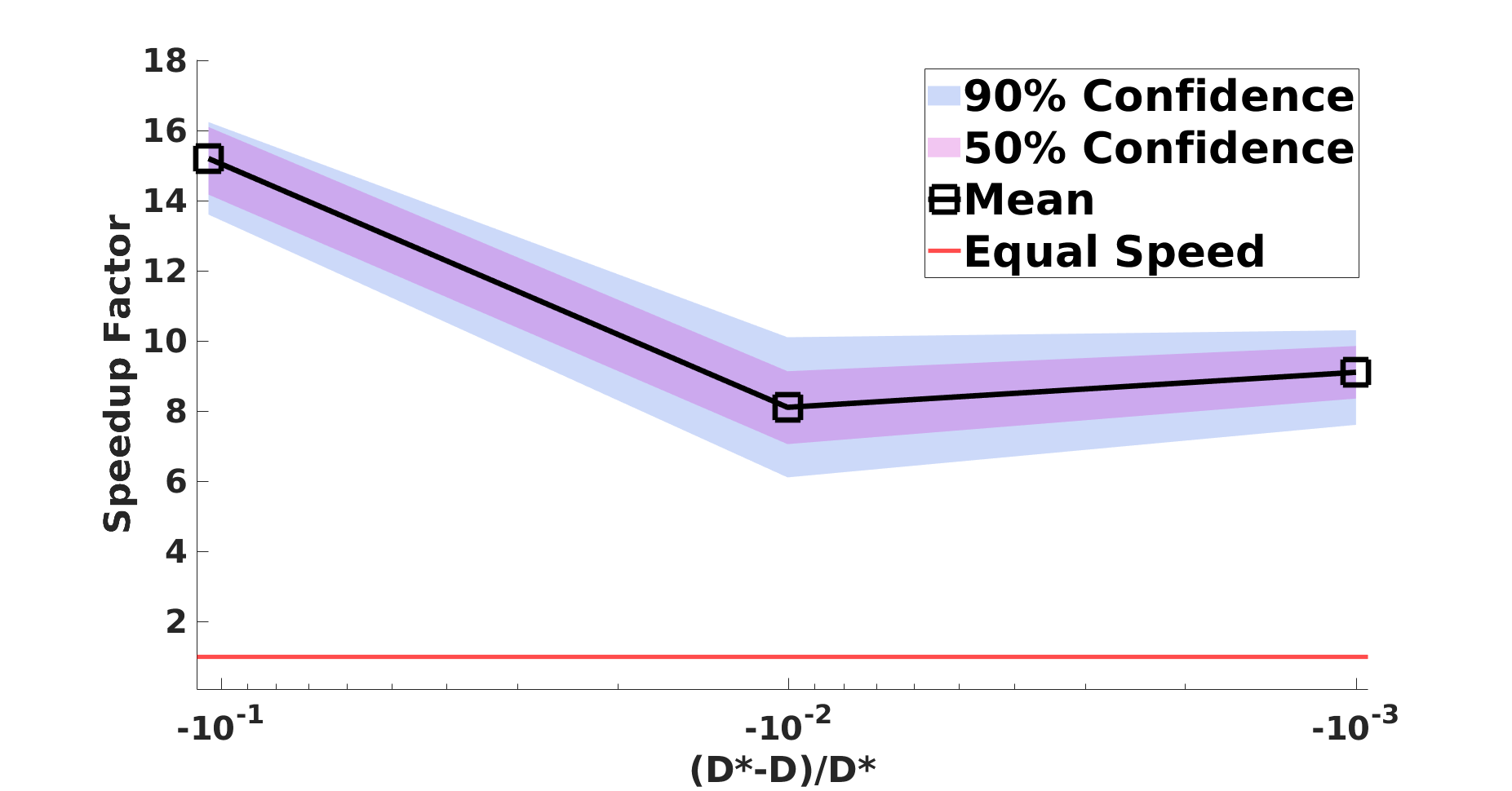}%
\\
\hline
\end{tabular}
\caption{\label{fig:conf-interval}
Speed-up factors of \SPAM \wrt \TRWS and \MPLPPP with confidence intervals for the different datasets.
The $x$-axis shows the normalized dual value and the $y$-axis the speed-up to achieve the same dual. The statistics are computed over all instances in a dataset. We show asymmetric confidence intervals with the equal percentage around the mean.
}
\end{figure*}

\end{document}